\documentclass[lettersize,journal]{IEEEtran}
\usepackage{amsmath,amsfonts}
\usepackage{algorithmic}
\usepackage{algorithm}
\usepackage{array}
\usepackage{textcomp}
\usepackage{stfloats}
\usepackage{url}
\usepackage{verbatim}
\usepackage{graphicx}
\usepackage{cite}
\usepackage{microtype}
\usepackage{graphicx}
\usepackage{subfigure}
\usepackage{booktabs} 
\usepackage{color}
\usepackage{wrapfig}
\usepackage{hyperref}
\usepackage{bbm}
\usepackage{tcolorbox}

\usepackage{amsmath}
\usepackage{amssymb}
\usepackage{mathtools}
\usepackage{amsthm}
\usepackage{multicol}
\usepackage{multirow}
\usepackage[capitalize,noabbrev]{cleveref}

\theoremstyle{plain}
\newtheorem{theorem}{Theorem}[section]

\theoremstyle{definition}

\theoremstyle{remark}

\usepackage{lettrine}

\begin{document}

\title{Provable Filter for Real-world Graph Clustering}

\author{Xuanting Xie, Erlin Pan, Zhao Kang, ~\IEEEmembership{Member,~IEEE,} Wenyu Chen and Bingheng Li\thanks{This work was supported by the National Natural Science
Foundation of China under Grant No. U24A20323; in part by Sichuan Provincial Department of Science and Technology, Natural Science Fund for Innovative Research Group Project (2024NSFTD0033). \textit{Corresponding author: Zhao Kang.}}
\thanks{X. Xie, Z. Kang, W. Chen are with the School of Computer Science and Engineering, University of Electronic Science and Technology of China, Chengdu, China; E. Pan is with Alibaba Group; B. Li is with Michigan State University;  (e-mail: x624361380@outlook.com; wujisixsix6@gmail.com; \{zkang, cwy\}@uestc.edu.cn; bingheng86@gmail.com).}}

\markboth{Journal of \LaTeX\ Class Files,~Vol.~14, No.~8, August~2021}%
{Shell \MakeLowercase{\textit{et al.}}: A Sample Article Using IEEEtran.cls for IEEE Journals}


\maketitle

\begin{abstract}
Graph clustering, an important unsupervised problem, has been shown to be more resistant to advances in Graph Neural Networks (GNNs). Most existing methods focus on homophilic graphs and ignore heterophily. This significantly limits their applicability, since real-world graphs exhibit a structural disparity and cannot simply be classified as homophily and heterophily. To fill this gap, we provide a principled way to handle practical graphs. Interestingly, we find that most homophilic and heterophilic edges can be correctly identified on the basis of neighbor information. Motivated by it, we construct two graphs that are highly homophilic and heterophilic, respectively. They are used to build low-pass and high-pass filters to capture holistic information. We make the first attempt to provide a theoretical analysis connecting the relationship between filters and clustering performance. Important features are further enhanced by the squeeze-and-excitation block. We validate our approach through extensive experiments on both homophilic and heterophilic graphs, as well as a co-saliency detection application. Empirical results demonstrate the superiority of our method compared to state-of-the-art clustering methods. In particular, our method achieves an average accuracy improvement of 1.82\% on heterophilic graphs and 0.83\% on homophilic graphs compared to the best baselines, including RGSL and DGCN.
\end{abstract}


\section{Introduction}
\IEEEPARstart{D}{ue} to the \textcolor{black}{rapid} expansion of graph data, there has been a \textcolor{black}{growing} interest in attributed graph clustering \cite{kang2024cdc}. This interest arises from the observation that many real-world graphs exhibit \textcolor{black}{locally heterogeneous} edge distributions, leading to node clusters. Graph clustering has shown \textcolor{black}{considerable} value in several applications, including data exploration \cite{fortunato2016community,lv2025toward}, visualization \cite{cui2008geometry}, anomaly detection \cite{perozzi2018discovering}, and feature discovery \cite{li2025ermav}. However, it has \textcolor{black}{remained largely unaffected} by advances in Graph Neural Networks (GNNs) \cite{tsitsulin2023graph,li2025unveiling}. A common approach relies on the graph autoencoder \cite{tian2014learning,chen2025adaptive} to generate node embeddings, which are subsequently fed into traditional clustering algorithms. Recent developments have produced several variants, including adversarial methods \cite{ARVGA,chen2024gress} and generative methods \cite{cheng2021multi}. Contrastive learning is also widely employed to enhance the \textcolor{black}{discriminative power of representations}. These methods typically predefine positive and negative pairs, \textcolor{black}{maximizing similarity among positives and increasing dissimilarity between positives and negatives} \cite{CCGC, SCGC, shen2025heterophily}. Finally, some shallow methods generate a new graph \cite{MCGC, shen2024beyond}, which is then \textcolor{black}{used by clustering algorithms to derive clusters}.

Existing methods encounter two \textcolor{black}{fundamental yet critical} problems. The first issue is \textcolor{black}{heterophily} \cite{guo2025disentangling}. Most methods assume that \textcolor{black}{homophily} is a key graph property, where connected nodes belong to the same cluster, while ignoring \textcolor{black}{heterophilic graphs}, in which connected nodes belong to different clusters. Heterophilic graphs are \textcolor{black}{common in practice} \cite{H2GCN,fangbenefits}. Methods designed for homophilic graphs are generally ineffective, and \textcolor{black}{stacked MLPs can even outperform many GNNs in heterophilic settings} \cite{H2GCN}. In practice, graphs inherently contain both homophilous and heterophilous neighbors, displaying \textcolor{black}{structural heterogeneity} \cite{EERM}. Consequently, methods that capture only low-frequency information in homophilic graphs or high-frequency information in heterophilic graphs are limited, \textcolor{black}{inevitably leading to information loss}. This limitation \textcolor{black}{substantially restricts the practical applicability of current graph learning techniques}. The second issue is that most clustering methods rely solely on local graph convolution, \textcolor{black}{failing to capture global structural information} \cite{li2018deeper}. Local information aggregation becomes less effective when low-degree nodes have limited neighborhoods, while \textcolor{black}{global information propagation is crucial for heterophilic graphs} \cite{Global-hete}.

 \begin{figure}[t]
    \centering
    \includegraphics[width=1.\linewidth]{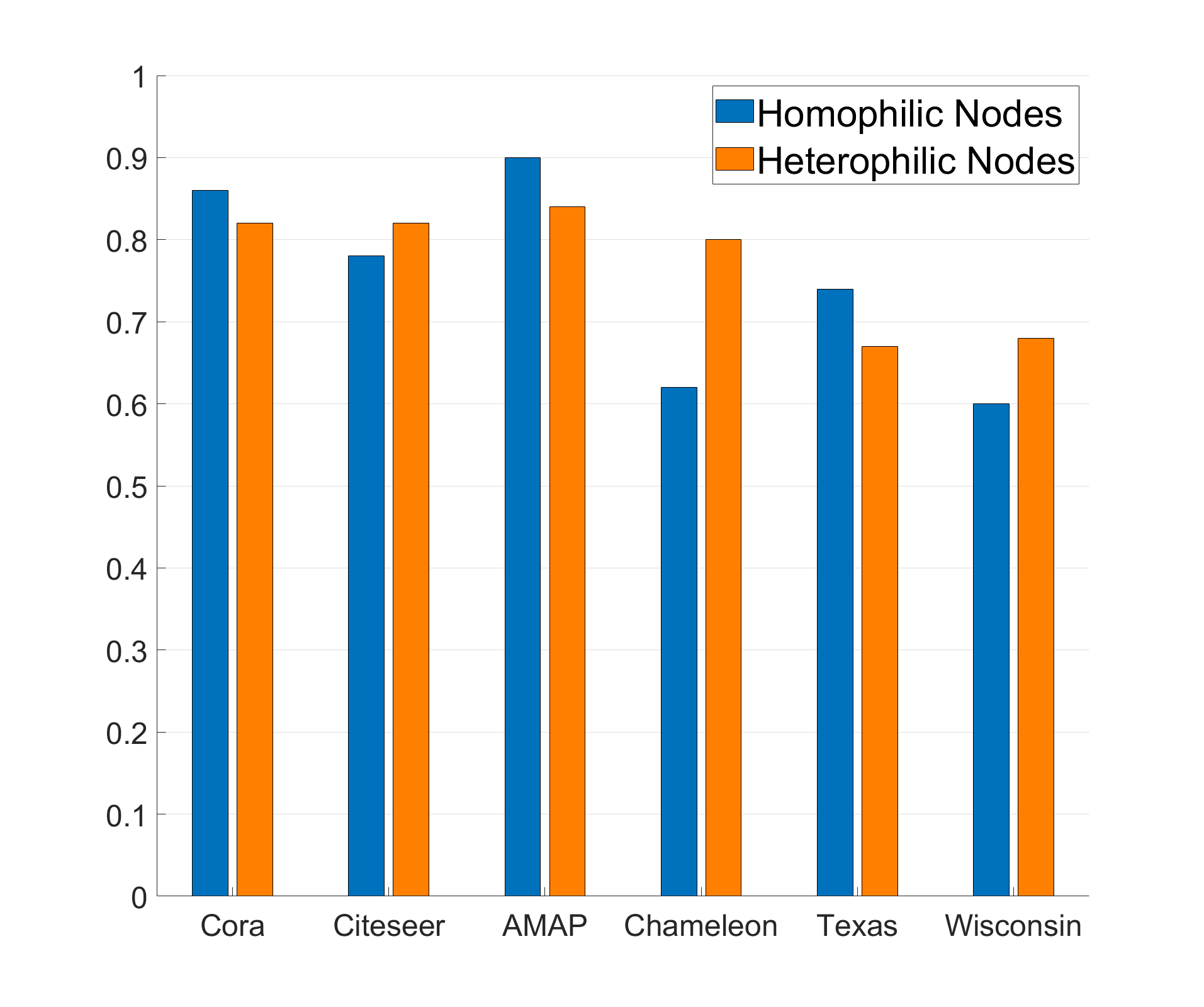}
    \caption{An interesting observation: most homophilic and heterophilic node pairs can be correctly identified by neighbor information.}
    \label{sthf1}
\end{figure}
To address the \textcolor{black}{aforementioned limitations}, we first investigate \textcolor{black}{neighbor commonality} in homophilic and heterophilic graphs through empirical experiments. Our intuition is inspired by \textcolor{black}{Balance Theory} \cite{cartwright1956structural,EvenNet}, which states: ``My enemy's enemy is my friend, and my friend's friend is also my friend." Thus, in heterophilic graphs, if two nodes share many common ``enemies," they are \textcolor{black}{highly likely to belong to the same cluster}; in homophilic graphs, nodes sharing many common ``friends" are also \textcolor{black}{likely to belong to the same cluster}. We conduct \textcolor{black}{empirical experiments} to verify this observation. Let two nodes be $v_i$ and $v_j$, with neighbors $\mathcal{N}_i$ and $\mathcal{N}_j$, respectively. We calculate the proportion of common neighbors as \textcolor{black}{$\frac{\mathcal{N}_i \cap \mathcal{N}_j}{\mathcal{N}_i \cup \mathcal{N}_j}$}. If $\frac{\mathcal{N}_i \cap \mathcal{N}_j}{\mathcal{N}_i \cup \mathcal{N}_j} \geq 0.5$, the nodes are classified as \textcolor{black}{homophilic pairs}; otherwise, they are treated as \textcolor{black}{heterophilic pairs}, as these nodes lack sufficient common ``friends" or ``enemies". We then compute the proportions of pairs correctly classified in six real homophilic and heterophilic datasets. Fig. \ref{sthf1} shows that most node pairs can be \textcolor{black}{accurately distinguished} using neighbor information. This \textcolor{black}{notable finding} motivates us to fully exploit neighbor information in real graph processing.

Based on this commonality, we first propose a simple method to construct two graphs: one highly homophilic and the other highly heterophilic. We then design a novel filter that accounts for their differing neighborhood sizes, which is \textcolor{black}{theoretically proven to improve clustering performance}. In addition, a squeeze-and-excitation block is employed to \textcolor{black}{enhance the most essential features}. Experiments on \textcolor{black}{graph tasks and co-saliency detection} demonstrate the versatility of our approach. \textcolor{black}{Different from existing graph clustering methods, our proposed graph filter is theoretically grounded through an analysis of clustering discriminability, scalable due to a hashing-based design, and further validated beyond clustering in a visual co-saliency detection task.} Our main contributions are summarized as follows.

\begin{itemize}
    \item{We identify the \textcolor{black}{neighbor commonality} in homophilic and heterophilic graphs, which provides an \textcolor{black}{unsupervised method for detecting heterophilic edges}. Based on this insight, we develop two unsupervised strategies for graph restructuring to capture both homophilic and heterophilic information from \textcolor{black}{arbitrary graphs}.}
    \item{We propose a novel filter for \textcolor{black}{real-world graph processing} and provide \textcolor{black}{theoretical evidence demonstrating its advantages}. To the best of our knowledge, we provide the first theoretical analysis \textcolor{black}{establishing the connection between graph filters and clustering performance}.}
    \item{We make the first attempt to apply squeeze-and-excitation idea in graph clustering to boost essential features after aggregation.}
    \item{Our method is verified through extensive experiments on 14 datasets, with an average accuracy improvement of 1.82\% on heterophilic graphs and 0.83\% on homophilic graphs.}
\end{itemize}

\section{Related work}
\subsection{Graph Clustering}
Attributed graph clustering methods can be broadly classified into \textcolor{black}{two categories}. The first category comprises \textcolor{black}{GNN-based methods}. These methods learn representations by employing a message-passing mechanism based on the \textcolor{black}{graph topology} \cite{GCN}. DAEGC \cite{DAEGC} is a goal-directed framework that combines an attention-based graph autoencoder with \textcolor{black}{deep latent representation learning}. MSGA \cite{MSGA} introduces a multiscale self-expression module to obtain more \textcolor{black}{discriminative coefficient representations} from each encoder layer, along with a self-supervised module to guide the learning process. SSGC \cite{SSGC} employs a modified Markov diffusion kernel to capture both global and local contexts of each node, enabling aggregation in large neighborhoods while \textcolor{black}{mitigating severe over-smoothing}. The concept of contrastive learning is also widely adopted in graph clustering. CCGC \cite{CCGC} addresses the semantic drift issue by using siamese encoders with unshared parameters and leveraging high-confidence cluster information to carefully select positive samples. SCGC \cite{SCGC} utilizes parameter unshared siamese encoders along with direct perturbation of node embeddings using Gaussian noise. The second kind is the shallow method without using neural networks. \cite{lin2023multi} acquires smooth embeddings by using low-pass filters. FGC \cite{FGC} and MCGC \cite{MCGC} adopt the low-pass filter and learn the nearest neighbor information, respectively, for the new graph. \textcolor{black}{FPGC \cite{xie2025one} proposes to learn a model for each node. THESAURUS \cite{deng2025thesaurus} leverages semantic prototypes to provide contextual information.} However, these methods focus only on homophilic graphs and ignore heterophily, which is limited, since real-world graphs always contain heterophilic edges.

\subsection{Learning on Heterophilic Graphs}
Heterophilic graphs present significant challenges for many GNN-based methods, \textcolor{black}{resulting in performance degradation}. The heterophilic problem has been extensively investigated in node classification tasks. Several methods have been proposed to \textcolor{black}{expand the receptive field to capture homophilic neighbors}. MixHop \cite{Mixhop} addresses heterophily by iteratively blending feature representations of neighbors at different distances. GloGNN \cite{Global-hete} constructs a graph incorporating high-order neighbors to distinguish homophilic neighbors \textcolor{black}{within} the graph structure. In addition, many methods have been proposed to \textcolor{black}{modify the message-passing architecture for heterophilic graphs}. ACM-GCN \cite{ACM-GCN} adaptively employs aggregation, diversification, and identity channels in each layer to \textcolor{black}{mitigate harmful heterophily and enhance GNN performance}. LINKX \cite{LINKX} separates graph and feature representations. Flow2GNN \cite{huang2024flow2gnn} employs a bidirectional message propagation mechanism that \textcolor{black}{decouples and redistributes heterophilic signals across both structural and attribute domains. PCConv\cite{li2024pc} combines both homophilic and heterophilic graph filter through Possion-Charlier polynomials.}

Although these approaches mitigate the heterophilic problem to some extent, they heavily depend on prior knowledge, such as labels, which are not accessible in unsupervised tasks. To our knowledge, SELENE \cite{SELENE}, CGC \cite{CGC}, and DGCN \cite{DGCN} are the only graph clustering methods that consider heterophily. SELENE \cite{SELENE} uses a dual-channel feature embedding pipeline to discriminate r-ego networks using node attributes and structural information separately. \textcolor{black}{HGDA \cite{fang2025homophily} proposes a multi-filter alignment framework that separately captures and aligns graph and attribute signals to improve cross-graph node classification.} DGCN and CGC use an adaptive filter to capture meaningful low- and high-frequency information. However, these methods are based on traditional low-pass filters and ignore global structure information. The most related work to ours is DGCN, which also constructs a homophilic and a heterophilic graph for clustering. However, it has the following drawbacks with respect to our method: 1) It needs $O(N^4)$ computational complexity to build graphs, which makes it impossible to apply on even medium-sized graphs. 2) Its homophilic graph is solely based on features without taking into account the property of the original structure. 3) Its filter fails to incorporate global structure information. 

\section{Methodology}
\textbf{Notations.} We define an undirected graph $\mathcal{G}=(\mathcal{V},E,X)$, where $\mathcal{V} = \left\{v_1, v_2,..., v_N \right\}$ represents the node set and $E$ denotes the edge set with $|E|$ edges. $X\in\mathbb{R}^{N \times d}$ is the feature matrix and $d$ is the number of channels. The matrix $A\in\mathbb{R}^{N \times N}$ signifies the adjacency matrix, where $A_{i j}= 1$ if $(v_i, v_j) \in E$; otherwise, $A_{i j}= 0$. Furthermore, the matrix $D$ corresponds to the degree matrix with $D_{i i}=\sum_j A_{i j}$. The normalized $A$ is $\tilde{A} = D^{-\frac{1}{2}}(A+I)D^{-\frac{1}{2}}$. The normalized Laplacian graph is $ L = I - \tilde{A} = U \Lambda U^{\top}$ and its eigenvalue matrix is $\Lambda=diag(\lambda_{1},\lambda_{1},\dots,\lambda_{N})$, where $0\le\lambda_{1}\le\lambda_{2}\le\dots\le\lambda_{N}\le2$. $U=\{u_{1},\dots,u_{N}\}$ indicates corresponding orthonormal eigenvectors.
 
\subsection{Graph Restructuring}
Graph clustering faces two critical challenges: real-world graphs contain a mixture of homophilic and heterophilic edges, and \textcolor{black}{the level of graph homophily is unknown a priori during clustering}. Directly using the original graph for filtering can \textcolor{black}{adversely affect downstream clustering}. Therefore, it is crucial to develop a \textcolor{black}{principled method for clustering graphs with varying levels of homophily}. In this paper, we adopt a graph restructuring approach that \textcolor{black}{separately extracts homophilic and heterophilic information}.

\subsubsection{Homophilic Graph Construction}


Based on our previous finding regarding neighbor commonality, we leverage neighborhood information to \textcolor{black}{differentiate homophilic and heterophilic edges}. Specifically, we employ \textcolor{black}{cosine similarity} to compute the distance between nodes in \textcolor{black}{both attribute and topology spaces}:

\begin{equation}
\begin{aligned}
& K_{i j}=\frac{X_{i,:}^\top \cdot X_{j,:}}{\left|\left|X_{i, :}\right|\right| \left|\left|X_{j, :}\right|\right|} \\
& B_{i j}=\frac{A_{i, :}^\top \cdot A_{j, :}}{\left|\left|A_{i, :}\right|\right| \cdot\left|\left|A_{j, :}\right|\right|} \\
&i, j \in[1,2, \cdots N],
\end{aligned}
\end{equation}
where $X_{i,:}$ is the $i$-th row of $X$. Then homophilic graph $M$ is constructed as follows:
\begin{equation}
\begin{aligned}
 &M=K \odot B,\\
&M_{ij}=\left\{
        \begin{aligned}
			1,&     & \text { if} \quad  M^2_{ij}\geq \epsilon. \\
			0,&     &\text { otherwise} .
		\end{aligned}
		\right.\\
\end{aligned}
\end{equation}
Here, $\odot$ denotes the Hadamard product, which is used to \textcolor{black}{identify common neighbors in both attribute and topology spaces}. Squaring $M_{ij}$ transforms the similarity score into a non-negative confidence measure, \textcolor{black}{amplifying strong homophilic relations and suppressing weak or noisy correlations}, thereby enabling a more stable and robust thresholding strategy. The threshold $\epsilon$ is set to 0.001 or 0.05 to \textcolor{black}{suppress noise}.

\subsubsection{Heterophilic Graph Construction}
We use the complementary graph idea to build heterophilic graph $G$ as follows:
\begin{equation}
\begin{aligned}
& \bar{K}=1 .-K \\
& \bar{M}=1 .-M \\
& G=\bar{K} \odot \bar{M},
\end{aligned}
\end{equation}
where \(\mathbf{1}\) denotes the matrix of all ones. $G$ characterizes \textcolor{black}{nodes with similar attributes that are distant in the topology space}. We use $M$ instead of $A$ because the neighbors in $M$ are more likely to be \textcolor{black}{homophilic}. $G$ can be dense, so we retain only the top five edges for each node, corresponding to the \textcolor{black}{five most distant nodes}.

\begin{figure*}[t]
    \centering
    \includegraphics[width=1.\linewidth]{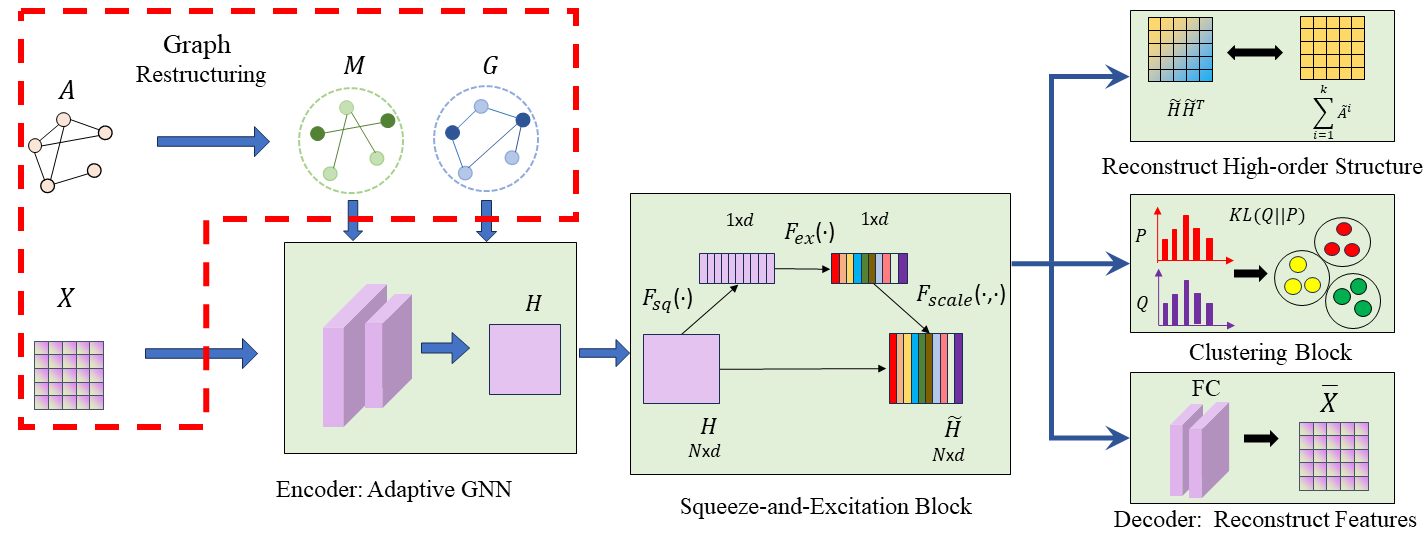}
    \caption{The overall architecture of our proposed method. The process begins with Graph Restructuring to construct homophilic ($M$) and heterophilic ($G$) graphs. These are processed by an Adaptive GNN encoder and enhanced via a Squeeze-and-Excitation (SE) block. Finally, the model minimizes a joint objective function involving high-order structure reconstruction, $KL$ divergence, and feature reconstruction.}
    \label{sth02}
\end{figure*}

\subsection{Clustering Framework}
Our overall framework for the graph clustering task is illustrated in Fig. \ref{sth02}. It first encodes node features using a novel filter applied to our constructed graphs $G$ and $M$. The learned representation is then enhanced via a squeeze-and-excitation block, consisting of two steps: \textcolor{black}{squeeze and excitation operations}. A decoder is applied to reconstruct the original features $X$. Finally, a clustering block is applied to \textcolor{black}{further enhance clustering performance}.

\subsubsection{Graph Filtering}
Previous research finds that low-frequency filters have a positive
correlation with homophily, while high-frequency filters have a negative correlation \cite{CGC,fang2022structure}. Treating the graphs $M$ and $G$ differently, we perform homophilic aggregation and heterophilic aggregation, respectively.

\textbf{Homophilic aggregation: } Traditional GNNs \cite{GCN} only aggregate local information, which loses information. Thus, we consider the global filter: $F = \exp (g(M))$, where $g(\cdot)$ is a linear filter that can capture low- or high-frequency information. To capture low-frequency information, we adopt GCN's filter kernel, i.e., for graph $M$, we have $g(M) = I-\tilde{L}_M = \tilde{M}$, where $\tilde{L}_M=U_1diag\{ \lambda_1^{(1)}, \lambda_2^{(1)},..., \lambda_N^{(1)} \}U_1^{\top}$ is Laplacian matrix of $\tilde{M}$ and $\tilde{M}$ is the normalized $M$. Then the global GNN becomes:
\begin{equation}
\begin{aligned}
F = \exp (\tilde{M}).
\end{aligned}
\end{equation}
The reason for using this global low-pass filter can be explained by the Taylor expansion as follows.

\begin{equation}
\exp (\tilde{M})=\sum_{n \geq 0} \frac{\tilde{M}^n}{n !}.
\end{equation}

The matrix can capture global information since $\exp (\tilde{M})_{i j} \neq 0$ \textcolor{black}{whenever a path with distinct hops connects nodes $v_i$ and $v_j$}. A high value of $\tilde{M}_{ij}$ indicates that over $n$ hops, nodes $v_i$ and $v_j$ exhibit \textcolor{black}{strong similarity}. The factorial division ensures the convergence of the infinite sum by preserving \textcolor{black}{mathematical stability through factors that decay rapidly to zero}. Consequently, the \textcolor{black}{significant weights decrease substantially with increasing hop distance}. This behavior is consistent with the characteristics of homophilic graphs, as the probability of encountering homophilic nodes decreases with increasing hop distance \cite{Global-hete}.

\textbf{Heterophilic aggregation: }To capture high-frequency information, we use traditional local GNN \cite{hp1}:
\begin{equation}
F = U_2 \operatorname{diag}\left(\left\{\frac{\lambda_i}{\lambda_{N}}\right\}_{i=1}^N\right) U_2^{\top} = \frac{1}{\lambda_{N}}\tilde{L}_G = \frac{2}{3}\tilde{L}_G,
\end{equation}
where $\tilde{L}_G=U_2diag\{ \lambda_1^{(2)}, \lambda_2^{(2)},..., \lambda_N^{(2)} \}U_2^{\top}$ is restructured graph $G$'s normalized Laplacian matrix. $\frac{1}{\lambda_{N}}$ is always set to $\frac{2}{3}$ according to previous work \cite{hp1}.

\textbf{Adaptive GNN: } Both homophilic and heterophilic neighboring nodes exist in \textcolor{black}{their respective graph types}. Consequently, we combine the aforementioned low- and high-pass filters and \textcolor{black}{define the aggregation operation at the $l$-th layer as follows}:

\begin{equation}
\begin{aligned}
H^{(l)} & = (1-\mu) U_1 \operatorname{diag}\left(\left\{e^{1-\lambda_i^{(1)}}\right\}_{i=1}^N\right) U_1^{\top} H^{(l-1)} W^{(l-1)}\\
& +\mu \left(\frac{2}{3}\tilde{L}_G\right) H^{(l-1)} W^{(l-1)}.
\end{aligned}
\end{equation}
The global filter is Min-Max normalized to \textcolor{black}{match} the range of the local filter, and $\mu$ is a trade-off parameter that balances low- and high-frequency information. $H^{(0)}$ is the original \textcolor{black}{features} $X$. With this aggregation, our model can adaptively \textcolor{black}{aggregate} low- and high-frequency information in each layer. The reason why we use global aggregation on the homophilic graph while using local aggregation on the heterophilic graph is analyzed in \textcolor{black}{Theorem \ref{pro}}.

\subsubsection{Squeeze-and-excitation Block}
After encoding, we \textcolor{black}{feed} the obtained node representation $H$ into the squeeze-and-excitation (SE) block, which is an attention mechanism based on attribute dimensions. \textcolor{black}{Important} node features will be improved. It consists of two major steps: \textcolor{black}{squeeze and excitation}.

\textbf{Squeeze:} To consider global features in latent space, the squeeze operation compresses the node representation $N \times d$ into $1 \times d$. Specifically, squeeze operation is global pooling as follows:
\begin{equation}
s=F_{s q}\left(H\right)=\frac{1}{N} \sum_{i=1}^{N} H_{i, :},
\end{equation}
where $s\in\mathbb{R}^{1 \times d}$, which is considered to be a channel-wise statistic that squeezes global feature information.

\textbf{Excitation:} Excitation follows the squeeze operation and aims to capture \textcolor{black}{the} dependencies between channels. Excitation is a simple gating mechanism, \textcolor{black}{enabling} flexibility in learning non-linear interactions between channels. It uses the sigmoid function $\sigma$ to excite the squeezed feature map. First, the dimension \textcolor{black}{is reduced} through the fully connected layer. Then, we use a ReLU function $\delta\left(\right)$ and a fully connected layer with increasing dimensionality to \textcolor{black}{restore} dimensions, that is, $W_2 \delta\left(W_1 s\right)$. Finally, the sigmoid function is applied \textcolor{black}{to the output}. The operation is summarized as follows:

\begin{equation}
\begin{aligned}
\tilde{s} & =F_{e x}\left(s, W\right) =\sigma\left(W_2 \delta\left(W_1 s\right)\right).
\end{aligned}
\end{equation}

\textbf{Reweight:} The result of the excitation operation is considered \textcolor{black}{as} the importance of each attribute due to the selection of the attribute dimension. Next, we \textcolor{black}{recalibrate} the representations by multiplying each attribute dimension by the initial representations $H$. The specific \textcolor{black}{re-weighting} operation is defined as follows:

 \begin{equation}
\tilde{H}=F_{ scale }\left(H, \tilde{s}\right)=\tilde{s} H,
\end{equation}
where $\tilde{H}$ is the output of SE block.

\subsubsection{Clustering Module}
Existing works often reconstruct the original topology structure \cite{DAEGC}, which \textcolor{black}{forces} neighboring nodes to have similar representations. However, this is based on the \textcolor{black}{assumption of homophily}. Reconstructing a heterophilic graph would \textcolor{black}{bring} highly dissimilar nodes close. As noted by \cite{Global-hete}, homophilic nodes exist among \textcolor{black}{multi-hop} neighbors. Thus, we propose reconstructing the \textcolor{black}{higher-order} topology structure as follows.

\begin{equation}
\mathcal{L}_{HS}=\frac{1}{N^2}\left\| \tilde{H}\tilde{H}^\top - \sum_{i=1}^{k}\tilde{A}^i\right\|_F^2,
\end{equation}
where $k$ indicates reconstructing $k$-order structure.

The decoder is applied to reconstruct the original features $X$. Some \textcolor{black}{``easy''} nodes show few feature variations during reconstruction, suggesting their trivial contribution to model training. We adopt a fully connected (FC) layer as the decoder with output $\bar{X}$ and \textcolor{black}{use} Scaled Cosine Error (SCE) \cite{GraphMAE} as the objective function:

\begin{equation}
    \mathcal{L}_{\mathrm{RE}}=\sum^N_{i=1 }\left(1-\frac{X_{i,:}^\top \bar{X}_{i,:}}{\left\|X_{i,:}\right\| \cdot\left\|\bar{X}_{i,:}\right\|}\right)^2.
\end{equation}

Finally, we utilize a clustering block for cluster enhancement. The soft assignment distribution $P$ is calculated as:
\begin{equation}
    p_{i j}=\frac{\left(1+\left\|\tilde{H}_{i,:}-c_j\right\|^2 / \beta\right)^{-\frac{\beta+1}{2}}}{\sum_{j^{\prime}}\left(1+\left\|\tilde{H}_{i,:}-c_{j^{\prime}}\right\|^2 / \beta\right)^{-\frac{\beta+1}{2}}},
\end{equation}
where cluster centers $c$ are initialized using \textcolor{black}{k-means} on the representations, and $\beta$ represents the degree of freedom in the Student’s $t$-distribution. The target distribution $Q$ is computed as:

\begin{equation}
    q_{i j}=\frac{p_{i j}^2 / \sum_i p_{i j}}{\sum_{j'}\left(p_{i j'}^2 / \sum_i p_{i j'}\right)}.
\end{equation}
	
We improve the cohesiveness of the cluster by bringing the node representation closer to the cluster centers by minimizing the KL divergence between the distributions $P$ and $Q$, which is calculated as:
	\begin{equation}
		\mathcal{L}_{CLU}=K L(Q \| P)=\sum_i \sum_j q_{i j} \log \frac{q_{i j}}{p_{i j}}
	\end{equation}
	Eventually, the objective function of our proposed Provable Filter for Graph Clustering (PFGC) method is formulated as:
	\begin{equation}
		\mathcal{L} = \mathcal{L}_{RE} + \gamma_1\mathcal{L}_{HS} + \gamma_2\mathcal{L}_{CLU},
	\end{equation}
	where $\gamma_1$ and $\gamma_2$ are trade-off parameters to balance three terms. We minimize the objective function above to train the model, obtaining the clustering label for node $i$ as:
	\begin{equation}
		z_{i} = \underset{j}{\operatorname{argmax}p_{ij}}.
	\end{equation}

\subsection{Computational Complexity}
PFGC's complexity comes mainly from graph restructuring and \textcolor{black}{the adaptive GNN}. The complexity of cosine similarity is $O(N^2)$. Adaptive GNN \textcolor{black}{requires} the eigendecomposition of the graph Laplacian. Empirically, there are two approaches to \textcolor{black}{improve} its speed. Firstly, the graph's eigendecomposition results can be stored and accessed after a one-time computation, \textcolor{black}{so} we only require a single eigendecomposition. Second, although the filtering process occurs in the spectral domain, empirically, the polynomial approximation can \textcolor{black}{approximate} this process. Consequently, the computational complexity of adaptive GNN is comparable to that of standard linear GCN \cite{GCN} and graph diffusion models \cite{grand}.

We apply the simhash technique to reduce the cost of graph restructuring. 
Let $\overrightarrow{h_i}$ be a random vector passing through the origin. The simhash function for feature $x$ is:

\begin{equation}
\operatorname{simhash}_{\overrightarrow{h_i}}(x)=\operatorname{sgn}(\langle{x}, \overrightarrow{h_i}\rangle),
\end{equation}

where $\operatorname{sgn}(\mathbf{a})= \begin{cases}1 & \text { if } \mathbf{a}>0 .\\ -1 & \text { if } \mathbf{a} \leq 0.\end{cases}$

Then $\mathbf{k}-\operatorname{simhash}$ is defined as:

\begin{equation}
\begin{aligned}
&\mathbf{k}-\operatorname{simhash}(x) \\
&=\left[\operatorname{simhash}_{\overrightarrow{h_1}}(x), \operatorname{simhash}_{\overrightarrow{h_2}}(x), \ldots, \operatorname{simhash}_{\overrightarrow{h_\mathbf{k}}}(x)\right]
\end{aligned}
\end{equation}

Simhash can approximate the Cosine similarity and the angle $\theta$ between $x_1$ and $x_2$ is defined as:

\begin{equation}
\begin{aligned}
\theta & =\left(1-P\left[\operatorname{simhash}_{\overrightarrow{h}}\left(x_1\right)=\operatorname{simhash}_{\overrightarrow{h}}\left(x_2\right)\right]\right) \times \pi \\
& \approx\left(1-\frac{\sum_{i} \textbf{1}\left[\operatorname{simhash}_{\overrightarrow{h_i}}\left(x_1\right)=\operatorname{simhash}_{\overrightarrow{h_i}}\left(x_2\right)\right]}{\mathbf{k}}\right) \times \pi,
\end{aligned}
\end{equation}
where \textbf{1} is an indicator function, and $\mathbf{k}$ is set to 8 or 16. The complexity of constructing $\mathbf{k}-\operatorname{simhash}$ is $O(\mathbf{k}dN)$ for all nodes, which is \textcolor{black}{significantly smaller} than $O(N^2)$. Note that the computational complexity of \textcolor{black}{SimHash} can be further reduced through GPU parallel processing techniques. Finally, we consider nodes that share the same $\mathbf{k}-\operatorname{simhash}$ as \textcolor{black}{homophilic edges}, while the top 5 far nodes are \textcolor{black}{treated as heterophilic edges}.

 \subsection{Theoretical Analysis of Filter Behaviors}
For graph clustering, the theoretical connection between \textcolor{black}{a filter} and clustering performance remains underexplored. To fill this gap, we theoretically analyze how to design the filters to better suit \textcolor{black}{practical graphs} with different homophily levels. The characteristics of homophilic and heterophilic graphs require \textcolor{black}{different} neighborhood sizes for filtering. We consider a local filter with a small receptive field and a global filter with a large receptive field. Assume that the graph is balanced and undirected, \textcolor{black}{with} $N$ nodes and $C$ clusters, \textcolor{black}{each containing} $N/C$ nodes. Its Laplacian matrix is $L$ with maximum eigenvalue $\lambda_N$. Consider global filters: $h_1(L) = \exp (I - L)$, $h_3(L) = \exp (L)$, and local filters that are widely applied in other methods: $h_2(L) = I - \frac{1}{\lambda_N}L$, $h_4(L) = \frac{1}{\lambda_N} L$. Global filters are Min-Max normalized to have the same range as local filters, \textcolor{black}{consistent with our proposed method}. Following \cite{Geom-GCN}, $r$ is defined as the homophily ratio: $r = \frac{1}{N}\sum\limits_{v \in\mathcal{V}}  \frac{\left|\{u \in \mathcal{N}_{v}|\ell(u)=\ell(v)\}\right|}{\left|\mathcal{N}_{v}\right|}$, where $\ell(v)$ indicates the label of node $v$. Then, we have the following theorem:

\begin{theorem} \label{pro}
Assume that low-pass and high-pass filters are applied on the graphs with $r>\frac{1}{C}$ and $r<\frac{1}{C}$, respectively. Then the clusters would be more discriminative with $h_1(L)$ compared to $h_2(L)$, while $h_4(L)$ improves the discriminativeness of the clusters more than $h_3(L)$.
\label{discriminative}
\end{theorem}


\begin{proof} 
$x$ is the original graph signal. $\bar{x}^{(i)}$ is the filtered representation by $h_i(L)$ (i=1,2,3,4). Let $S_{in}^{(i)}$ and ${S}_{out}^{(i)}$ be the total intra-cluster and inter-cluster distance; $\bar{S}_{in}^{(i)}$ and $\bar{S}_{out}^{(i)}$ be the average intra-cluster and inter-cluster distance of $x^{(i)}$. $\mathbb{E}\left[S^{(i)}\right] = \mathbb{E}\left[S_{out}^{(i)} - S_{i n}^{(i)}\right]$ and $\mathbb{E}\left[\bar{S}^{(i)}\right] = \mathbb{E}\left[\bar{S}_{out}^{(i)} - \bar{S}_{i n}^{(i)}\right]$. We can represent signal $x$ as the linear combination of the eigenvectors:
\begin{equation}
    x=\sum_{t=1}^N a_t u_t,
    \label{12}
\end{equation}
where $a_t = u_t^\top x$ is the coefficient. Then the filtered signal can be calculated as:
\begin{equation}
    \begin{aligned}
    \bar{x}&=h(L) x=U h(\Lambda) U^{\top} x \\
    & =\left(\sum_{t=1}^N h\left(\lambda_t\right) u_t u_t^{\top}\right)\left(\sum_{t=1}^N a_t u_t\right)=\sum_{t=1}^N h \left( \lambda_t \right) a_t u_t.\\
    \end{aligned}
    \label{13}
\end{equation}
Filtered signals with $h_1(L)$, $h_2(L)$, $h_3(L)$, $h_4(L)$ are $\bar{x}^{(1)}$, $\bar{x}^{(2)}$, $\bar{x}^{(3)}$, $\bar{x}^{(4)}$, respectively. The smoothness of the neighboring nodes can be calculated as:
\begin{equation}
\begin{gathered}
\sum_{i, j \in E}\left(x_i-x_j\right)^2=x^{\top} L x \\
\sum_{i, j \in E}\left(\bar{x}_i-\bar{x}_j\right)^2=x^{\top} h(L) x \\
\end{gathered}
\label{15}
\end{equation}
The eigenvector's smoothness score and its corresponding eigenvalue are equivalent:
\begin{equation}
    \lambda_t=u_t^{\top} L u_t=\sum_{i, j \in E}\left(u_{t, i}-u_{t, j}\right)^2\\
\end{equation}

We have:
\begin{equation}
\begin{aligned}
    &\sum_{(i, j) \in E}\left(\bar{x}_i-\bar{x}_j\right)^2 \\
    &=\left(\sum_{t=1}^N h\left(\lambda_t\right) a_t u_t^{\top}\right)\left(\sum_{t=1}^N \lambda_t u_t u_t^{\top}\right)\left(\sum_{t=1}^N h\left(\lambda_t\right) a_t u_t\right) 
    \\
    & =\sum_{t=1}^N a_t^2 \lambda_t h^2\left(\lambda_t\right) \\
\end{aligned}
\end{equation}
Since $x$ is an arbitrary unit graph signal, all $a_t$'s are independently identically distributed (i.i.d.). For i.i.d. random variables $a_i$ and $a_j$ ($j>i$), we have:
\begin{equation}
\begin{aligned}
    & \mathbb{E}\left(a_i^2\right)=\mathbb{E}\left(a_j^2\right) \\
    & \Rightarrow \lambda_i \mathbb{E}\left[a_i^2\right]=\mathbb{E}\left[\lambda_i a_i^2\right] \leqslant \mathbb{E}\left[\lambda_j a_j^2\right]=\lambda_j \mathbb{E}\left[a_j^2\right]
\end{aligned}
\end{equation}
Then, we can calculate the total distance between neighboring nodes after filtering. $\Delta d$ is the distance between the global and local filters. Specifically, the distance between $x^{(1)}$ and $x^{(2)}$ can be calculated as:

\begin{equation}
\begin{aligned}
    \mathbb{E}[\Delta d] &=\mathbb{E}\left[\sum_{i, j \in E}\left(\bar{x}^{(1)}_i-\bar{x}^{(1)}_{j}\right)^2\right]-\mathbb{E}\left[\sum_{i, j \in E}\left(\bar{x}^{(2)}_i-\bar{x}^{(2)}_j\right)^2\right] \\
& =\mathbb{E}\left[\sum_{t=1}^N a_t^2 \lambda_t h_1^2\left(\lambda_t\right)\right]-\mathbb{E}\left[\sum_{t=1}^N a_t^2 \lambda_t h_2^2\left(\lambda_t\right)\right] \\
& =\sum_{t=1}^N\left\{\left[h_1^2\left(\lambda_t\right)-h_2^2\left(\lambda_t\right)\right] \lambda_t \mathbb{E}\left[a_t^2\right]\right\} \\
& =\sum_{t=1}^N\left\{\left[\left(\frac{e^{1-\lambda_t}-e^{1-\lambda_{N}}}{e^{1-\lambda_{1}}-e^{1-\lambda_{N}}}\right)^2-\left(1-\frac{\lambda_t}{\lambda_{N}} \right)^2\right] \lambda_t \mathbb{E}\left(a_t^2\right)\right\} \\
\end{aligned}
\end{equation}
Note that our global filter is Min-Max normalized to ensure it has the same range as local filter. Define a general function $g_0(\lambda)$ regarding $\lambda$:
\begin{equation}
\begin{aligned}
&g_0(\lambda)=\left(\frac{e^{1-\lambda}-e^{1-\lambda_{N}}}{e^{1-\lambda_{1}}-e^{1-\lambda_{N }}}\right)^2-\left(1-\frac{\lambda}{\lambda_{N}} \right)^2 \\
& =\left(\frac{e^{1-\lambda}-e^{1-\lambda_{N}}}{e^{1-\lambda_{1}}-e^{1-\lambda_{N}}}+1-\frac{\lambda_t}{\lambda_{N}} \right)\left(\frac{e^{1-\lambda}-e^{1-\lambda_{N}}}{e^{1-\lambda_{1}}-e^{1-\lambda_{N}}}-1+\frac{\lambda}{\lambda_{N}} \right) \\
&\lambda \in \left\{\lambda_1,\lambda_2,\cdots,\lambda_N\right\}
\label{26}
\end{aligned}
\end{equation}
Define $g_1(\lambda)=\frac{e^{1-\lambda}-e^{1-\lambda_{N}}}{e^{1-\lambda_1}-e^{1-\lambda_N}}+1-\frac{1}{\lambda_N} \lambda$ and $g_2(\lambda)=\frac{e^{1-\lambda}-e^{1-\lambda_{N}}}{e^{1-\lambda_{1}}-e^{1-\lambda_{N}}}-1+\frac{1}{\lambda_N} \lambda$. Obviously, $g_1(\lambda)$ is monotonically decreasing with respect to $\lambda$, so $g_1(\lambda)\geq g_1\left(\lambda_{N}\right)=0$. Therefore, we only need to focus on $g_2(\lambda)$ to judge whether $g_0(\lambda)$ is positive or negative. Let $g^{\prime}_2(\lambda)$=0, we have:
\begin{equation}
\begin{aligned}
\bar{\lambda}=1+\ln \left(e^{1-\lambda_{1}}-e^{1-\lambda_N}\right)+\ln \lambda_N > 0
\end{aligned}
\end{equation}
Thus $g_2(\lambda)$ is monotonically decreasing on $(0, \bar{\lambda})$ and monotonically increasing on $(\bar{\lambda}, +\infty)$. Considering $g_2(\lambda_N)=0$, $g_0(\lambda)=0$ only has one root $\lambda_0$ in $(\lambda_1, \lambda_N)$. Thus $g_0(\lambda)$$ \geq 0$ when $\lambda_t \in \left[\lambda_1, \lambda_m \right]$ and $g_0(\lambda)$$ \le 0$ when $\lambda_t \in \left[\lambda_{m+1}, \lambda_N \right]$, where $\lambda_m$ is the closest one to $\lambda_0$ in $\lambda_t \le \lambda_0$. Note that both filters are applied on the reconstructed homophilic graph, which indicates $\lambda_1$ is very close to 0. Thus, $\lambda_0$ and $g_0(\lambda)_{max}=(2-\frac{\lambda_1}{\lambda_N})\frac{\lambda_1}{\lambda_N}$ are very close to 0. Therefore, we assume that $\left\|\left[g_0\left(\lambda_1\right), \cdots ,g_0\left(\lambda_m\right)\right]^{\top}\right\|_2^2 \leq\left\|\left[g_0\left(\lambda_{m+1}\right), \ldots ,g_0\left(\lambda_N\right)\right]^{\top}\right\|_2^2$. Then we have the following:

\begin{equation}
\begin{aligned}
& 0 \leqslant \sum_{t=1}^m \left[h_1^2\left(\lambda_t\right)-h_2^2\left(\lambda_t\right)\right] \leqslant \sum_{m+1}^N \left[h_2^2\left(\lambda_t\right)-h_1^2\left(\lambda_t\right)\right] \\
\Rightarrow & 0 \leqslant \sum_{t=1}^m\left[h_1^2\left(\lambda_t\right)-h_2^2\left(\lambda_t\right)\right] \lambda_t \\
&<\lambda_m \sum_{t=1}^m\left[h_1^2\left(\lambda_t\right)-h_2^2\left(\lambda_t\right)\right] \\
& \leq \lambda_{m+1} \sum_{m+1}^N\left[h_2^2\left(\lambda_t\right)-h_1^2\left(\lambda_t\right)\right] \\
&<\sum_{m+1}^N\left[h_2^2\left(\lambda_t\right)-h_1^2\left(\lambda_t\right)\right] \lambda_t \\
\end{aligned}
\end{equation}

We can derive the following result from it:

\begin{equation}
\begin{aligned}
    & \sum_{t=1}^N\left\{\left[h_1^2\left(\lambda_t\right)-h_2^2\left(\lambda_t\right)\right] \lambda_t\right\}<0 \\
    & \Rightarrow \mathbb{E}[\Delta d] = \sum_{t=1}^N\left\{\left[h_1^2\left(\lambda_t\right)-h_2^2\left(\lambda_t\right)\right] \lambda_t \mathbb{E}\left[a_t^2\right]\right\}<0
\end{aligned}
\end{equation}
Similarly, we can prove the following conclusion with high-pass filters:
\begin{equation}
    \begin{aligned}
& \mathbb{E}\left[\Delta d^{\prime}\right] \\
&=\mathbb{E}\left[\sum_{i, j \in E}\left(\bar{x_i}^{(3)}-\bar{x_{j}}^{(3)}\right)^2\right]-\mathbb{E}\left[\sum_{i, j \in E}\left(\bar{x_i}^{(4)}-\bar{x_j}^{(4)}\right)^2\right] \\
&<0
\end{aligned}
\end{equation}
Define $z_i$ as the label of node $i$. We can compute the total intra-cluster and inter-cluster distance as follows:
\begin{equation}
\begin{aligned}
& \mathbb{E}\left[S_{i n}^{(1)} - S_{i n}^{(2)}\right] = \mathbb{E}\left[r \Delta d\right] \\
& =\mathbb{E}\left[\sum_{\substack{i, j \in E \\ z_i=z_j}}\left[\left(\bar{x}_i^{(1)}-\bar{x}_j^{(1)}\right)^2-\left(\bar{x}_i^{(2)}-\bar{x}_j^{(2)}\right)^2\right]\right] \\
&\mathbb{E}\left[S_{out }^{(1)} - S_{out }^{(2)}\right] = \mathbb{E}\left[(1-r) \Delta d\right]\\
&=\mathbb{E}\left[\sum_{\substack{i, j \in E \\
z_i \neq z_j}}\left[\left(\bar{x}_i^{(1)}-\bar{x}_j^{(1)}\right)^2-\left(\bar{x}_i^{(2)}-\bar{x}_j^{(2)}\right)\right]\right]
\end{aligned}
\end{equation}
Finally, the average distance between inter-cluster and intra-cluster nodes can be calculated as:
\begin{equation}
\begin{aligned}
&\mathbb{E}\left[\bar{S}^{(1)}-\bar{S}^{(2)}\right] \\
& =\mathbb{E}\left[(\bar{S}_{out}^{(1)}-\bar{S}_{out}^{(2)}) - (\bar{S}_{in}^{(1)}-\bar{S}_{in}^{(2)})\right] \\
& =\frac{2 C(1-r) \mathbb{E}\left[\Delta d\right]}{N^2(C-1)} - \frac{ 2 C r \mathbb{E}\left[\Delta d\right]}{N^2}\\
& =\frac{2 C}{(C-1) N^2}\left[\mathbb{E}\left[\Delta d\right](1 - C r)\right]
\end{aligned}
\end{equation}
Similarly, after high-pass filtering, we have:
\begin{equation}
\begin{aligned}
&\mathbb{E}\left[\bar{S}^{(3)}-\bar{S}^{(4)}\right] \\
& =\frac{2 C}{(C-1) N^2}\left[\mathbb{E}\left[\Delta d^{\prime}\right](1 - C r)\right]
\end{aligned}
\end{equation}
Obviously, $\frac{2 C}{(C-1) N^2} > 0$. We reach the following conclusions.

(1) If $r>\frac{1}{C}$, then $\mathbb{E}\left[\bar{S}^{(1)}-\bar{S}^{(2)}\right] > 0$, that is, $\mathbb{E}\left[\bar{S}^{(1)}\right]>\mathbb{E}\left[\bar{S}^{(2)}\right]$

(2) If $r<\frac{1}{C}$, then $\mathbb{E}\left[\bar{S}^{(3)}- \bar{S}^{(4)}\right] < 0$, that is, $\mathbb{E}\left[\bar{S}^{(3)}\right]<\mathbb{E}\left[\bar{S}^{(4)}\right]$
\end{proof}
 
Therefore, the clusters are more discriminative after applying a global filter to a homophilic graph, whereas they are more discriminative \textcolor{black}{when using} a local filter on a heterophilic graph. Connected nodes of a homophilic graph are more likely to \textcolor{black}{belong to} the same cluster, whereas the opposite is true for a heterophilic graph. Thus, a combination of global and local filters can make the clusters more \textcolor{black}{separable} by reducing intra-cluster distance and increasing inter-cluster distance. This explains why the proposed adaptive GNN is more suitable for \textcolor{black}{processing real-world graphs}.


\section{Experiments on Clustering}
\subsection{Datasets}
To ensure \textcolor{black}{comparability}, we select benchmark datasets that are widely used in graph clustering research. For homophilic graphs, we \textcolor{black}{consider} five benchmark datasets: Cora, CiteSeer, Pubmed \cite{Cora}, Amazon Photo (AMAP) \cite{AMAP}, and USA Air-Traffic (UAT) \cite{EAT}. For heterophilic graphs, we select six datasets. Chameleon and Squirrel are page-page networks covering specific topics and sourced from Wikipedia \cite{rozemberczki2021multi}. Webgraphs from multiple computer science departments at various universities were collected by Carnegie Mellon University\footnote{http://www.cs.cmu.edu/afs/cs.cmu.edu/project/theo-11/www/wwkb/}, including Cornell, Wisconsin, and Washington. Roman-Empire is derived from the Roman Empire article on English Wikipedia and is \textcolor{black}{selected} as one of the longest articles available on the platform \cite{Roman}. Additionally, we include two large-scale graph datasets: the image relationship network Flickr \cite{Flickr} and the social network Twitch-Gamers \cite{twitch}. The homophily ratio is computed following \cite{Geom-GCN}, where higher values indicate stronger homophily. The \textcolor{black}{key statistics} of these datasets are summarized in Table \ref{tab::datasets}.

 \begin{table}[t]
		\centering
	\caption{Statistics information of datasets.}
		\label{tab::datasets}%
  \resizebox{0.45\textwidth}{!}{
    \begin{tabular}{ccccccc}
    \toprule
    \multicolumn{2}{c}{Graph datasets} & Nodes & Dims. & Edges & Clusters & Homophily Ratio \\
    \midrule
 & Cornell & 183   & 1703  & 298   & 5     & 0.1220  \\
          & Wisconsin & 251   & 1703  & 515   & 5      & 0.1703  \\
          & Washington  & 230   & 1703  & 786   & 5    & 0.1434 \\
          & Chameleon & 2277 & 2325  & 31371 & 5    & 0.2299 \\
          & Squirrel & 5201  & 2089  & 217073 & 5    & 0.2234  \\
          & Roman-empire & 22662  & 300 & 32927 & 18    & 0.0469  \\
          & Flickr &89250 &500 &899756 & 7 &0.3195 \\
    \midrule
 & Cora  & 2708  & 1433  & 5429  & 7    & 0.8137  \\
          & Citeseer & 3327  & 3703  & 4732  & 6     & 0.7392  \\
          & Pubmed & 19717 &500 &44327 &3 &0.8024 \\
          & UAT   & 1190  & 239  & 13599 & 4     & 0.6978 \\
          & AMAP  & 7650  & 745   & 119081 & 8    & 0.8272 \\
          & Twitch-Gamers &168114 &7 &67997557 &2 &0.5453 \\
          & Ogbn-Arxiv &169343 &128 &1166243 &40 &0.6778 \\
    \bottomrule
    \end{tabular}}%
		
	\end{table}

 \subsection{Comparison Methods}
 To showcase the superiority of our PFGC, we compare it against 18 baselines for performance evaluation. These 18 baselines cover five distinct categories of methods: 1) traditional GNN-based methods, such as DAEGC \cite{DAEGC}, MSGA \cite{MSGA}, SSGC \cite{SSGC}, GMM \cite{CDRS}, RWR \cite{RWR}, ARVGA \cite{ARVGA}; 2) contrastive learning-based methods, such as MVGRL \cite{MVGRL}, SDCN \cite{SDCN}, DFCN \cite{DFCN}, DCRN \cite{DCRN}, SCGC \cite{SCGC}, and CCGC \cite{CCGC}, leveraging MLPs and GNNs together to learn an aligned representation from augmented perspectives; 3) the advanced clustering approach AGE \cite{AGE}, which achieves clustering-friendly embbedings through Laplacian smoothing filters and adaptive encoders; 4) shallow methods that utilize the filter to smooth raw features and reduce noise, including MCGC \cite{MCGC}, FGC \cite{FGC}, and CGC \cite{CGC}; (5) recent clustering methods that consider heterophily, including SELENE \cite{SELENE}, DGCN \cite{DGCN}, CDC \cite{kang2024cdc}, and RGSL \cite{RGSL}.

\begin{table*}[htbp]
  \centering
  \caption{Results on heterophilic graphs. The best results are \textbf{bolded} with \textcolor[rgb]{ 1,  0,  0}{\textbf{red}} and the second-best performance is also \textbf{bolded}. "OOM" indicates out of memory. All results are reported with mean$\pm$std under ten runs}
  \label{Reheter}
  
  \resizebox{\textwidth}{!}{%
    \begin{tabular}{ccccccccccccc}
    \toprule
    \multirow{2}[4]{*}{Methods} & \multicolumn{2}{c}{Cornell} & \multicolumn{2}{c}{Wisconsin} & \multicolumn{2}{c}{Washington} & \multicolumn{2}{c}{Chameleon} & \multicolumn{2}{c}{Squirrel}  & \multicolumn{2}{c}{Roman-empire}\\
    \cmidrule{2-13}          & ACC   & NMI   & ACC   & NMI   & ACC   & NMI   & ACC   & NMI   & ACC   & NMI   & ACC   & NMI \\
    \midrule
    DAEGC & 42.56  & 12.37 & 39.62  & 12.02 & 46.96  &17.03   & 32.06  & 6.45   & 25.55  & 2.36 &21.23 &12.67\\
    MSGA  & 50.77  & 14.05 & 54.72  & 16.28 & 49.38   & 6.38   & 27.98   &6.21   & 27.42  & 4.31 &19.31 &12.25\\
    FGC   & 44.10  & 8.6   & 50.19  & 12.92 & 57.39  & 21.38   & 36.50 &11.25   & 25.11  & 1.32 &14.46 &4.86 \\
    GMM  & 58.86 & -   & 52.08 & 8.89  & 60.86  & 20.56 & -   & -   & -   & - & -   & - \\
    RWR  & 58.29 & -   & 53.96 & 16.02 & 63.91  & 23.13 & -   & -   & -   & - & -   & -\\
    ARVGA  & 56.23  & -   & 56.23  & 13.73 & 60.87  & 16.19 & -   & -   & -   & - & -   & -\\
    DCRN &51.32 &9.05 &57.74 &19.86 &55.65 &14.15 &34.52 &9.11 &30.69 &6.84 &32.57 &29.50 \\
    SELENE &57.96&17.32 &71.69&39.51 &-&- &\textbf{38.97}&\textbf{20.63} &-&- &-&-\\

    CGC & 44.62 & 14.11 & 55.85 & 23.03 & 63.20 & 25.94 & 36.43 & 11.59 & 27.23 & 2.98 & 30.16 & 27.25\\
    DGCN & \textbf{62.29$_{\pm 0.85}$} & \textbf{29.93$_{\pm 0.78}$} & \textbf{71.71$_{\pm 1.02}$} & \textbf{41.29$_{\pm 1.15}$} & \textbf{69.13$_{\pm 0.94}$} & 28.22$_{\pm 0.88}$ & 36.14$_{\pm 1.21}$ & 11.23$_{\pm 1.09}$ & \textbf{31.34$_{\pm 0.97}$} & 7.24$_{\pm 0.82}$ & OOM & OOM\\
    CDC & 51.40 & 14.20 & 63.70 & 31.80 & 67.32 & 27.92 & 36.77 & 15.05 & 27.90 & 4.30 & 31.02 & 30.46 \\
    RGSL & 57.44 & 28.95 & 56.60 & 28.57 & 66.09 & \textbf{29.79} & 38.52 & 12.79 & 30.74 & \textcolor[rgb]{ 1,  0,  0}{\textbf{8.74}} & \textcolor[rgb]{ 1,  0,  0}{\textbf{34.57}} & \textbf{31.23} \\
    PFGC &\textcolor[rgb]{ 1,  0,  0}{\textbf{66.12$_{\pm 0.75}$}}&\textcolor[rgb]{ 1,  0,  0}{\textbf{33.04$_{\pm 0.62}$}} &\textcolor[rgb]{ 1,  0,  0}{\textbf{74.10$_{\pm 0.82}$}}&\textcolor[rgb]{ 1,  0,  0}{\textbf{47.53$_{\pm 0.65}$}} &\textcolor[rgb]{ 1,  0,  0}{\textbf{70.43$_{\pm 0.77}$}} &\textcolor[rgb]{ 1,  0,  0}{\textbf{37.99$_{\pm 0.71}$}} &\textcolor[rgb]{ 1,  0,  0}{\textbf{41.28$_{\pm 0.93}$}}&\textcolor[rgb]{ 1,  0,  0}{\textbf{21.62$_{\pm 0.87}$}}&\textcolor[rgb]{ 1,  0,  0}{\textbf{33.05$_{\pm 1.01}$}}&\textbf{7.58$_{\pm 0.64}$} &\textbf{33.98$_{\pm 1.25}$}&\textcolor[rgb]{ 1,  0,  0}{\textbf{38.80$_{\pm 0.87}$}}\\
    \bottomrule
    \end{tabular}%
  }
\end{table*}

\subsection{Experimental Setting}
To ensure fairness, all experimental settings in different data sets adhere to the DGCN \cite{DGCN}, which performs a grid search to find the best results\footnote{Our code is available at https://github.com/XieXuanting/PFGC}. Our network undergoes training with the Adam optimizer for 200 epochs until convergence. According to \cite{Global-hete}, more than half of the homophilic neighbors can be included in 5 hops. Thus, the hyper-parameter $k$ is fixed to 1 on graphs with more than 10K nodes to save computing time, and empirically set to 5 on other datasets. The optimizer's learning rate is set between 1e-2 and 1e-3. $\gamma_1$ and $\gamma_2$ are chosen to balance the three types of loss in [1e-3, 1e-2, 1e-1, 1, 10]. We tune the trade-off parameter $\mu$ from 0.1 to 0.7. ACCuracy (ACC) and Normalized Mutual Information (NMI) are adopted as clustering metrics.

For two large datasets, we pre-store the reconstructed graphs with the simhash technique. $\gamma_2$ is set to 0 to improve the training efficiency (i.e., removing $\mathcal{L}_{CLU}$). We apply the mini-batch training with the size 1000 to compute the loss function. The hidden dimension is set to 100.

\subsection{Results}
The clustering results on heterophilic graphs are presented in Table \ref{Reheter}. It can be observed that PFGC achieves \textcolor{black}{consistently strong performance} across most datasets. Specifically, our method outperforms SELENE, CGC, DGCN, CDC, and RGSL, all of which explicitly account for heterophily. SELENE is a GNN-based method, whereas both DGCN and CGC employ adaptive filtering mechanisms. This demonstrates that our adaptive filter can leverage graph information more effectively than these methods. CDC and RGSL also incorporate graph reconstruction strategies, suggesting that graph reconstruction is a widely adopted approach to enhance graph quality in heterophilic settings. Consequently, the reconstructed graph in our approach is expected to be both reliable and effective. Notably, these approaches focus exclusively on local information, whereas our method also accounts for the global cluster structure. Furthermore, PFGC achieves higher ACC compared to other methods, exceeding them by up to 9\% on Cornell, 16\% on Wisconsin, and 6\% on Washington. This confirms that explicitly accounting for heterophily enhances model performance on real-world datasets. Moreover, conventional GNN-based approaches exhibit suboptimal performance on heterophilic graphs, consistent with prior observations in the literature.

The clustering results on homophilic graphs are presented in Table \ref{Rehomo}. It can be observed that PFGC exhibits competitive performance, attaining the best results in most cases and ranking second in the remaining three instances. Notably, state-of-the-art contrastive learning methods demonstrate \textcolor{black}{considerable variability} in performance. This variability arises because their performance heavily relies on data augmentation strategies, which require domain knowledge and exhibit limited flexibility across datasets. For shallow methods, FGC and MCGC employ only a low-pass filter, resulting in performance that varies considerably across datasets. CGC, a shallow method with an adaptive filter, demonstrates superior performance compared to FGC and MCGC. This indicates that relying on a single filter is insufficient for real-world graphs exhibiting diverse homophily levels. PFGC and DGCN outperform AGE in most cases, suggesting that low-pass filtering alone is inadequate, and that high-frequency information plays a critical role in homophilic graphs. PFGC consistently outperforms DGCN across all datasets. This is partly because the new graphs in DGCN are constructed without considering the intrinsic structural properties, rendering DGCN less effective in graphs with strong homophily.


\begin{table*}[!htbp]
  \centering
  \caption{Results on homophilic graphs.}
  \label{Rehomo}%
  
  \resizebox{\textwidth}{!}{%
    \begin{tabular}{ccccccccccc}
    \toprule
    \multirow{2}[4]{*}{Methods} & \multicolumn{2}{c}{Cora} & \multicolumn{2}{c}{Citeseer}  & \multicolumn{2}{c}{Pubmed} & \multicolumn{2}{c}{UAT} & \multicolumn{2}{c}{AMAP}  \\
    \cmidrule{2-11}          & ACC   & NMI   & ACC   & NMI   & ACC   & NMI   & {ACC} & NMI   & ACC   & NMI   \\
    \midrule
    DFCN  & 36.33  & 19.36 & 69.50  & 43.9  &-&-& 33.61  & 26.49  & 76.88 & 69.21   \\
    DCRN  & 48.93  & -   & 70.86  & \textcolor[rgb]{ 1,  0,  0}{\textbf{45.86}} &-&-& -  &- & {\textcolor[rgb]{ 1,  0,  0}{\textbf{79.94}}} & \textcolor[rgb]{ 1,  0,  0}{\textbf{73.70 }}    \\
    SSGC  & 69.60  & 54.71 & 69.11  & 42.87 &-&- & 36.74  & 8.04 & {60.23} & 60.37  \\
    MVGRL & 70.47  & 55.57 & 68.66  & 43.66 &-&-& 44.16  & 21.53 & {45.19} & 36.89  \\
    SDCN  & 60.24  & 50.04 & 65.96  & 38.71 &65.78&29.47& 52.25  & 21.61 & {53.44} & 44.85   \\
    AGE   & 73.50  & \textbf{57.58} & 70.39  & 44.92 &-&-& 52.37  & 23.64 & {75.98} & -       \\
    MCGC  & 42.85  & 24.11 & 64.76  & 39.11 &66.95&32.45& 41.93  & 16.64 & 71.64 & 61.54   \\
    FGC   & 72.90  & 56.12 & 69.01  & 44.02&\textbf{70.01}&31.56 & 53.03  & 27.06 & 71.04 & -        \\
    SCGC & 73.88$_{\pm 0.88}$ & 56.10$_{\pm 0.72}$ &71.02$_{\pm 0.77}$ &\textbf{45.25$_{\pm 0.45}$} &67.73$_{\pm 0.45}$&28.65$_{\pm 0.67}$&\textbf{56.58$_{\pm 1.62}$} &28.07$_{\pm 0.71}$ &77.48$_{\pm 0.37}$ &67.67$_{\pm 0.88}$   \\
    CCGC & 73.88$_{\pm 1.20}$ & 56.45$_{\pm 1.04}$ & 69.84$_{\pm 0.94}$ & 44.33$_{\pm 0.79}$ &68.06$_{\pm 0.55}$&30.92$_{\pm 0.48}$&56.34$_{\pm 1.11}$ &\textbf{28.15$_{\pm 1.92}$} &77.25$_{\pm 0.41}$ & 67.44$_{\pm 0.48}$   \\
    CGC & \textbf{75.15} & 56.90 &69.31 &43.61 &67.43&\textbf{33.07}&49.58&17.49 &73.02&63.26 \\
    DGCN & 72.19$_{\pm 1.05}$ & 56.04$_{\pm 1.21}$ & \textbf{71.27$_{\pm 0.91}$} & 44.13$_{\pm 0.85}$ & OOM & OOM & 52.27$_{\pm 0.45}$ & 23.54$_{\pm 0.38}$ & 76.07$_{\pm 1.35}$ & 66.13$_{\pm 1.92}$ \\
    PFGC & \textcolor[rgb]{ 1,  0,  0}{\textbf{76.51$_{\pm 0.86}$}} & \textcolor[rgb]{ 1,  0,  0}{\textbf{58.25$_{\pm 0.71}$}} & \textcolor[rgb]{ 1,  0,  0}{\textbf{71.90$_{\pm 0.79}$}} & \textbf{45.45$_{\pm 0.56}$} &\textcolor[rgb]{ 1,  0,  0}{\textbf{72.89$_{\pm 0.32}$}} &\textcolor[rgb]{ 1,  0,  0}{\textbf{33.30$_{\pm 0.48}$}} & \textcolor[rgb]{ 1,  0,  0}{\textbf{56.81$_{\pm 0.34}$}} &\textcolor[rgb]{ 1,  0,  0}{\textbf{29.33$_{\pm 0.42}$}} & \textbf{78.50$_{\pm 0.45}$} & \textbf{70.81$_{\pm 0.74}$}  \\
    \bottomrule
    \end{tabular}%
  }
\end{table*}

\begin{figure}[t]
    \centering
    \includegraphics[width=.47\linewidth]{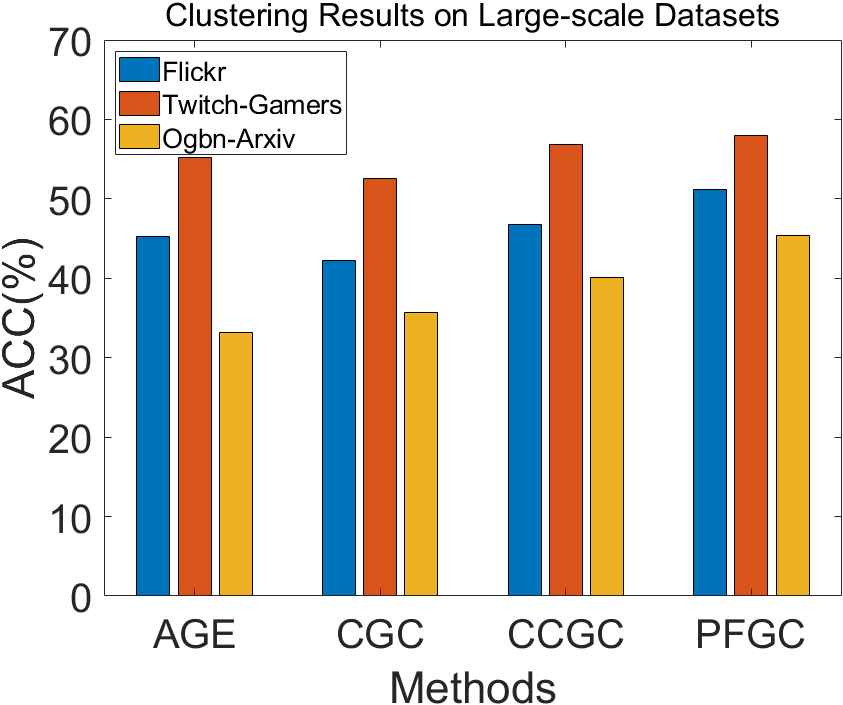}
    \includegraphics[width=.47\linewidth]{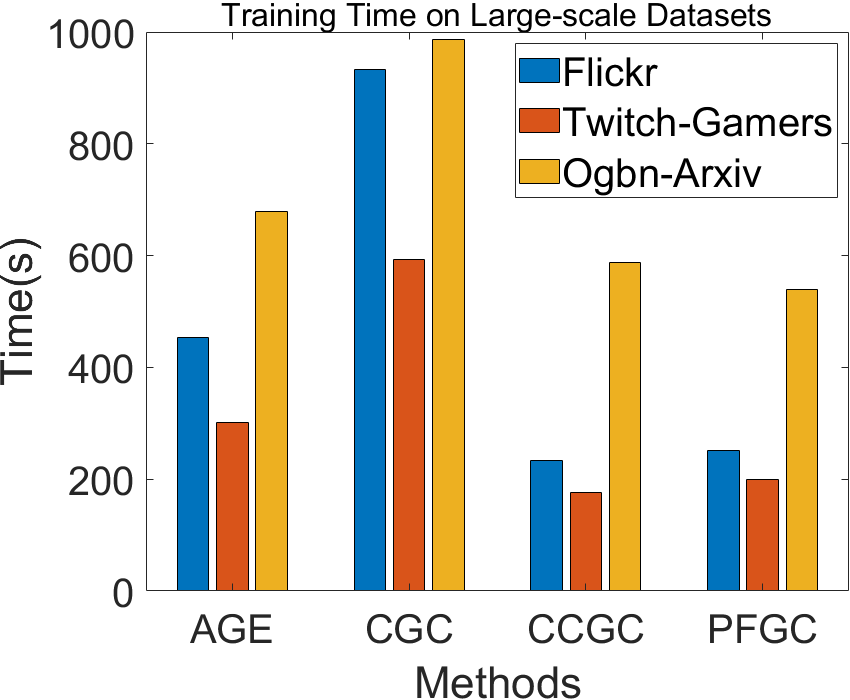}    
    \caption{Clustering results and running time on Flickr, Twitch-Gamers, and Ogbn-Arxiv.}
    \label{fig45}
\end{figure}

The \textcolor{black}{accuracy and training-time performance} on Flickr, Twitch-Gamers, and Ogbn-Arxiv are \textcolor{black}{presented} in Fig \ref{fig45}. It can be \textcolor{black}{observed that PFGC consistently outperforms all baseline methods} across the datasets. The training time of PFGC is comparable to that of CCGC, yet \textcolor{black}{substantially lower than that required by CGC and AGE}. These findings not only \textcolor{black}{demonstrate the efficacy of PFGC} but also \textcolor{black}{emphasize its scalability and computational efficiency for large-scale graphs}.

In summary, PFGC achieves stable and promising performance in both homophilic and heterophilic cases. This is primarily due to the ability of the constructed graph to capture the homophilic and heterophilic information inherent in the graphs. Therefore, PFGC is suitable for clustering practical graphs, even when the homophily ratios are unknown.
\begin{figure}[t]
		\centering
		\subfigure[Cora with restructured graph]{
			\includegraphics[width=0.46\linewidth]{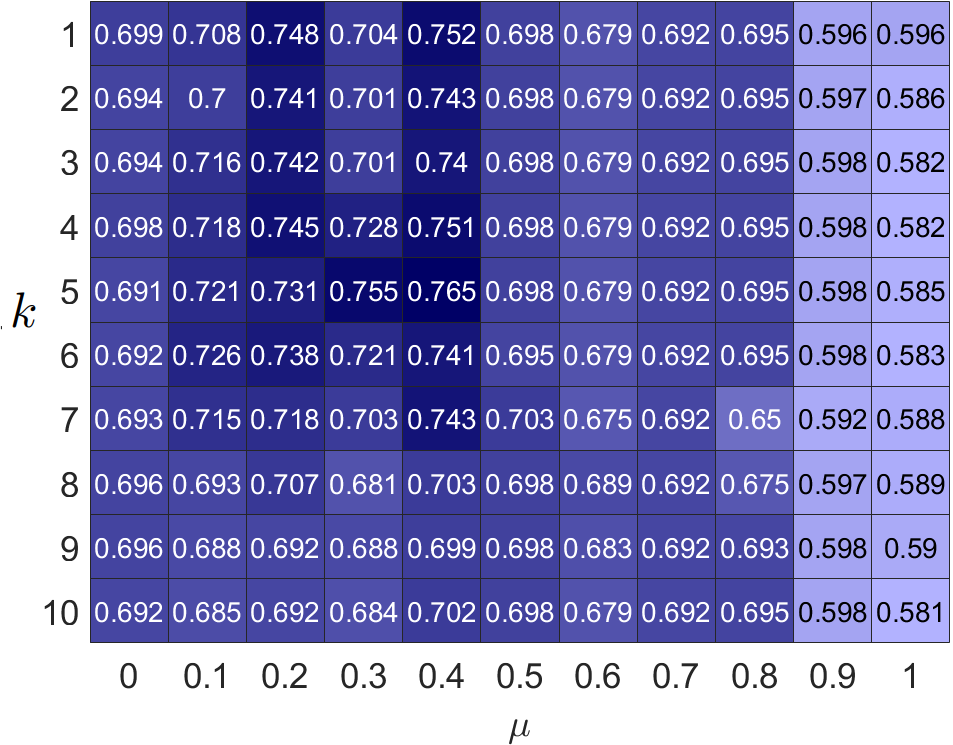}
		}
		\subfigure[Cora with raw graph $A$]{
			\includegraphics[width=0.46\linewidth]{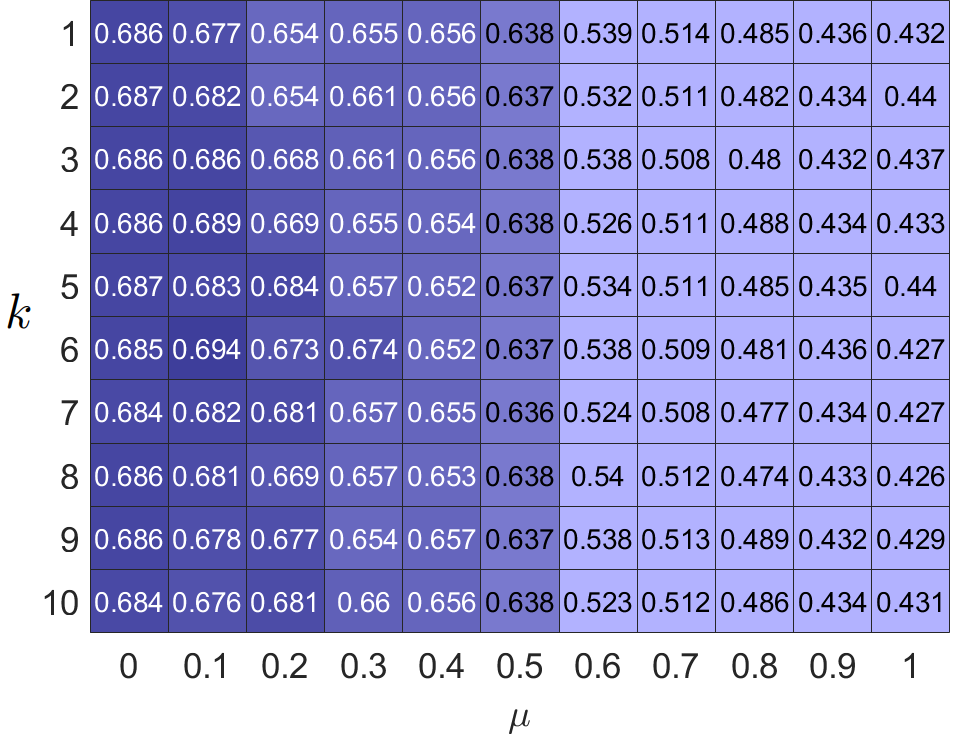}
		}
		\subfigure[Washington with restructured graph]{
			\includegraphics[width=0.46\linewidth]{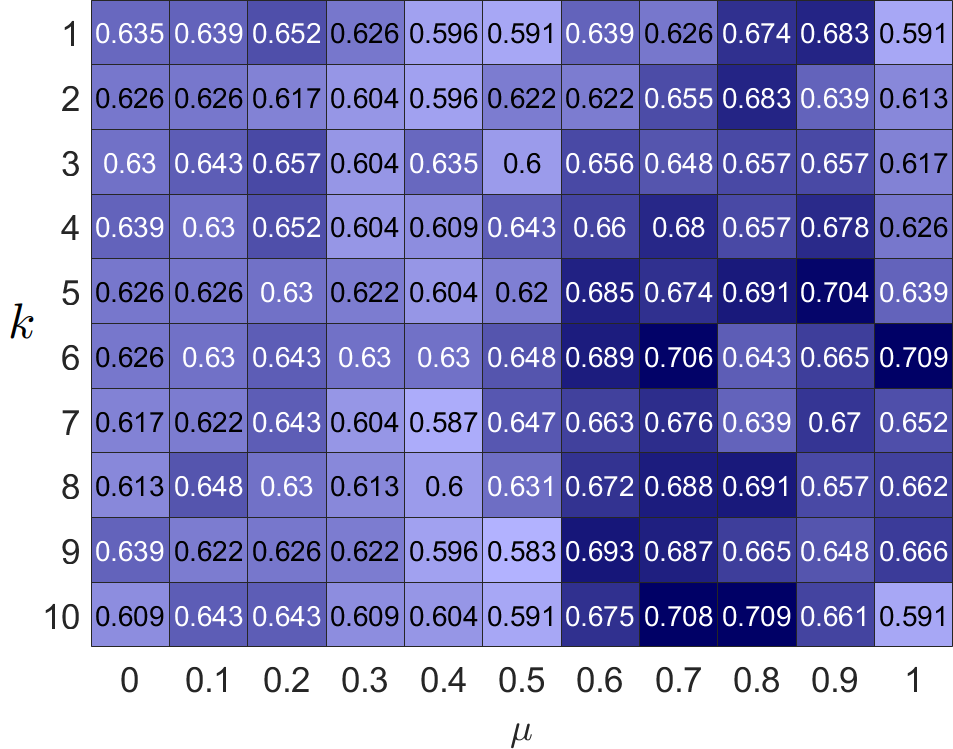}
		}
		\subfigure[Washington with raw graph $A$]{
			\includegraphics[width=0.46\linewidth]{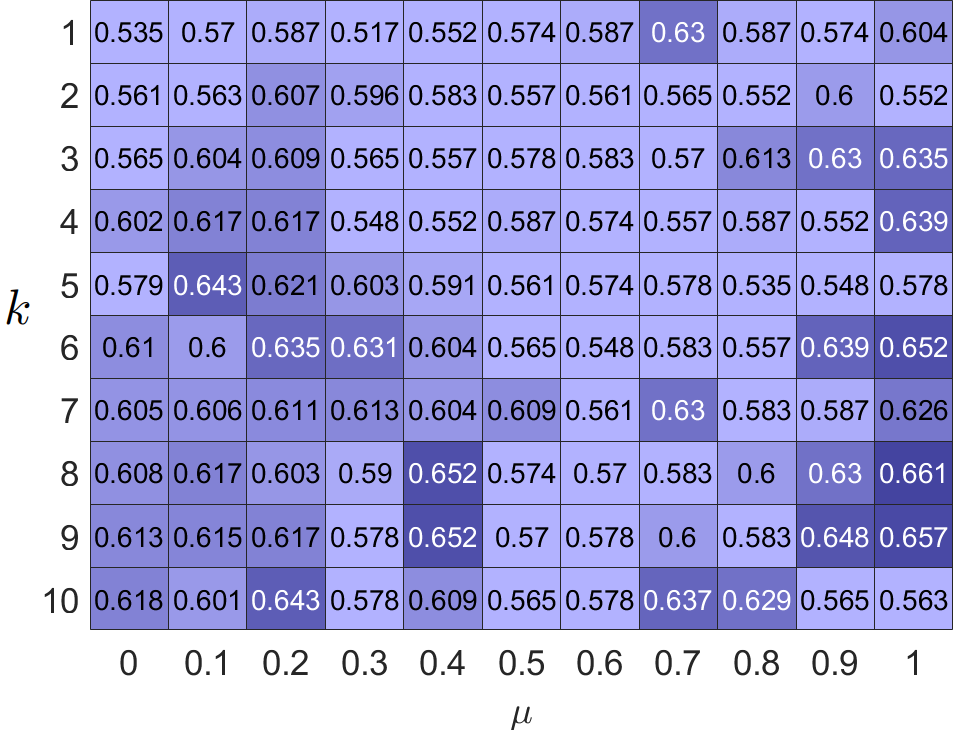}
		}
		\caption{The effect of $k$ and $\mu$ on Cora and Washington with restructured graphs and raw graphs.}
		\label{pa14}
\end{figure}
\begin{figure}[t]
		\centering

		\includegraphics[width=1.\linewidth]{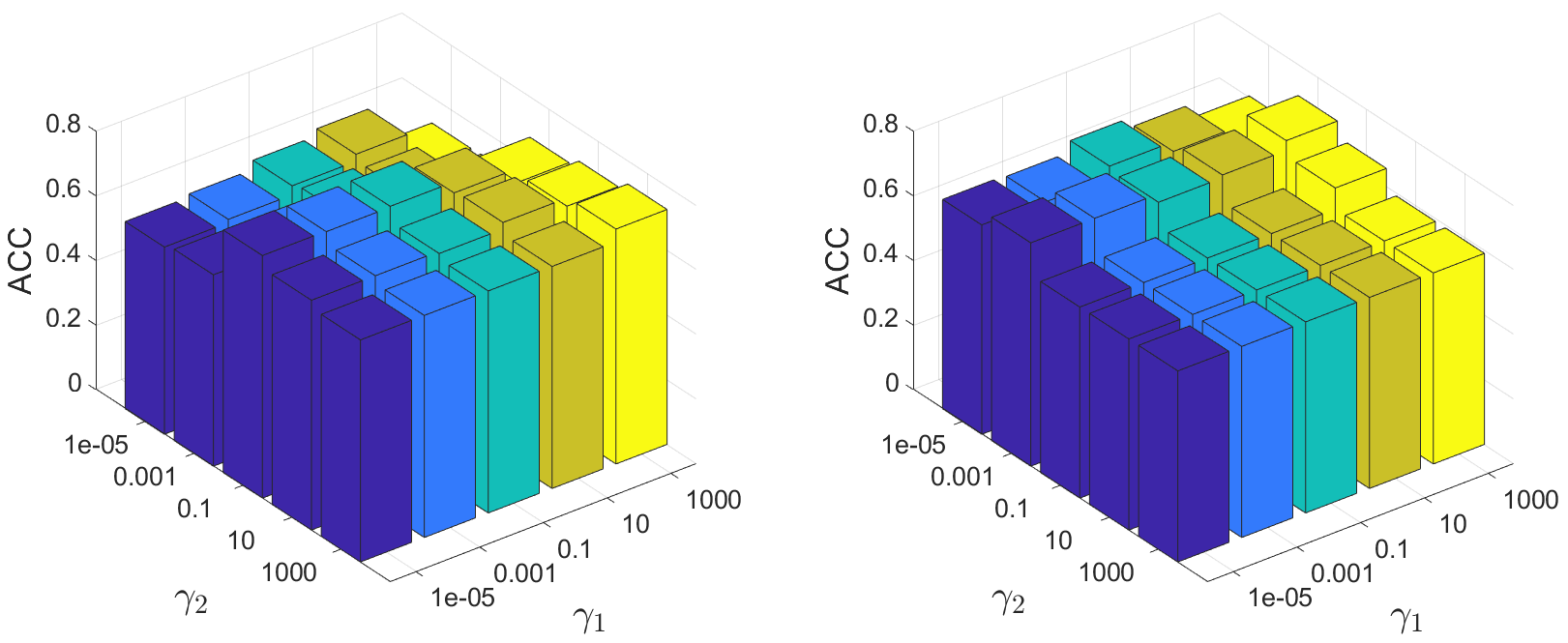}

		\caption{The effect of $\gamma_1$ and $\gamma_2$ on Cora (left) and Washington (right).}
		\label{pa5}
\end{figure}


\begin{table*}[h]
\centering
\caption{Results of ablation study. The best performance is marked in \textbf{bold}.}\label{abla1}
\begin{center}
\resizebox{1.\textwidth}{!}{
\begin{tabular}{ l rl| rl| rl| rl| rl| rl| rl}

\midrule
Methods & \multicolumn{2}{c}{{PFGC w/o SE}} & \multicolumn{2}{c}{{PFGC-1}} & \multicolumn{2}{c}{{PFGC-2}} & \multicolumn{2}{c}{{PFGC-3}} & \multicolumn{2}{c}{{PFGC-be}} & \multicolumn{2}{c}{{PFGC-as}} & \multicolumn{2}{c}{{PFGC}} \\ 
\cmidrule(lr){2-3} \cmidrule(lr){4-5} \cmidrule(lr){6-7} \cmidrule(lr){8-9} \cmidrule(lr){10-11}\cmidrule(lr){12-13}\cmidrule(lr){14-15}
    & {ACC} & {NMI} & {ACC} & {NMI}& {ACC} & {NMI} & {ACC} & {NMI}& {ACC} & {NMI} & {ACC} & {NMI}& {ACC} & {NMI}\\
\midrule

Cora

&75.18 &57.20
&75.88 &56.94
&76.07&57.91
&75.30&57.85
&70.81 &50.12
&70.04 &51.47
&\textbf{76.51} &\textbf{58.25}

\\
Citeseer
&71.38 & 44.92
& 67.53 & 41.20
& 69.14 & 42.89
& 65.62 & 40.15
&63.35 &40.04
&68.22 &43.11
&\textbf{71.90} &\textbf{45.45}
\\
Pubmed
& 72.46 & 32.81
& 69.88 & 30.75
& 71.24 & 31.93
& 68.15 & 29.42
&67.95 &30.88
&64.20 &29.25
&\textbf{72.89} &\textbf{33.30} 
\\
AMAP
&75.88 &67.90
&75.35 &67.54
&78.33&70.25
&76.26&67.82
&74.63 &66.90
&69.22 &62.32
&\textbf{78.50} &\textbf{70.81}
 \\
 Wisconsin
&72.51 &47.47
&72.11 &47.42
&70.92&43.59
&68.53&34.87
&69.39 &45.14
&67.33 &41.95
&\textbf{74.10} &\textbf{47.53}
\\
Washington
&68.53 &35.06
&67.21 &36.69
&64.78&33.86
&61.30&32.19
&66.50 &34.88
&62.13 &33.26  
&\textbf{70.43} &\textbf{37.99}
\\
Chameleon
& 40.75 & 21.18
& 39.42 & 20.35
& 38.16 & 19.09
& 36.53 & 17.74
&38.92 &20.38
&36.40 &19.04
&\textbf{41.28}&\textbf{21.62}
\\
Squirrel
& 30.68 & \phantom{0}5.27
& 31.50 & \phantom{0}6.79
& 30.15 & \phantom{0}6.65
& 28.88 & \phantom{0}5.42
&27.98 &\phantom{0}5.14
&29.43 &\phantom{0}6.76 
&\textbf{33.05} &\textbf{\phantom{0}7.58}
\\
\bottomrule
\end{tabular}}
\end{center}
\end{table*}

\begin{table}[]
\centering
\caption{Analysis of time (seconds) and GPU memory cost (GB, in the ``()'') on Cora and Squirrel. ``-'' indicates running on CPU.}
\begin{center}
\resizebox{0.49\textwidth}{!}{
\begin{tabular}{lccccc}
\hline Methods & AGE & RGSL &CCGC & DGCN  & AMLP \\
\hline Cora& 68.23(1.42) & 125.43(-) &65.43(1.28)& 97.42(1.18) & \textbf{60.32(1.12)} \\
\hline Squirrel& 123.52(3.87) & 258.92(-) &96.23(2.57)& 125.62(2.39) & \textbf{87.25(2.17)} \\
\hline
\end{tabular}}
    \label{time}
\end{center}
\end{table}

\subsection{Parameter Analysis}
Our filter operates on \textcolor{black}{recently restructured homophilic and heterophilic graphs}, in which neighboring nodes exhibit \textcolor{black}{intrinsic tendencies toward similarity or dissimilarity}. Furthermore, the loss function \textcolor{black}{incorporates a mechanism to reconstruct k-order structural information}. To assess the efficacy of restructured graphs and high-order structures, we \textcolor{black}{evaluate PFGC’s clustering accuracy over hyper-parameters $\mu$ and $k$ across various graph datasets}. The results are \textcolor{black}{illustrated} in Fig. \ref{pa14}. It can be observed that \textcolor{black}{graph restructuring substantially enhances model performance} on both homophilic and heterophilic graphs. Moreover, optimal performance is consistently achieved when $\mu \neq 0$, \textcolor{black}{highlighting the critical role of incorporating high-frequency information}. The homophilic graph attains peak performance at small $\mu$ values, whereas heterophilic graphs favor larger $\mu$, \textcolor{black}{implying that homophilic graphs emphasize low-frequency components, while heterophilic graphs are more sensitive to high-frequency components}. Additionally, performance \textcolor{black}{degrades sharply} for large $\mu$ on raw homophilic graphs; \textcolor{black}{thus, excessively separating neighboring nodes in homophilic graphs is inadvisable}. Concerning the hyper-parameter $k$, our model \textcolor{black}{selects smaller values for Cora and larger values for Washington}. This behavior arises because \textcolor{black}{similar nodes are proximal in homophilic graphs but tend to be more distant in heterophilic graphs}.

The loss function includes two balance parameters $\gamma_1$ and $\gamma_2$. To see their effect, we set $\gamma_1$, $\gamma_2$=[1e-5, 1e-3, 1e-1, 1e1, 1e3]. Their effect on clustering ACC of Cora and Washington is shown in Fig.\ref{pa5}. Our technique performs effectively for a wide range of parameters. The influence of $\gamma_2$ on performance is greater than $\gamma_1$, indicating the importance of cluster enhancement.


\subsection{Robustness Analysis}
To demonstrate the robustness of PFGC, we consider scenarios in which the graph structure is \textcolor{black}{corrupted} by noise. Specifically, we \textcolor{black}{randomly add} edges to the graphs while removing an equal number of original edges \cite{fu2022p}. We define the noise ratio as $Nr = \frac{\text{\#random edges}}{\text{\#all edges}}$, representing the \textcolor{black}{proportion} of random edges. With $Nr$ = \{0.2, 0.4, 0.6, 0.8\}, we \textcolor{black}{evaluate} the ACC of AGE, CGC, CCGC, and PFGC on Cora and Citeseer datasets. From Fig.\ref{rebnoise}, it can be observed that PFGC \textcolor{black}{consistently achieves superior} performance. In particular, when the perturbation rate is very high, our method \textcolor{black}{demonstrates greater stability compared to} other clustering approaches. This is because PFGC’s ability to capture both homophilic and heterophilic information, along with local and global insights, \textcolor{black}{ensures its adaptability to diverse} node relationships and structural complexities in graphs.

\begin{figure}[!htbp]
    \centering
    \includegraphics[width=.8\linewidth]{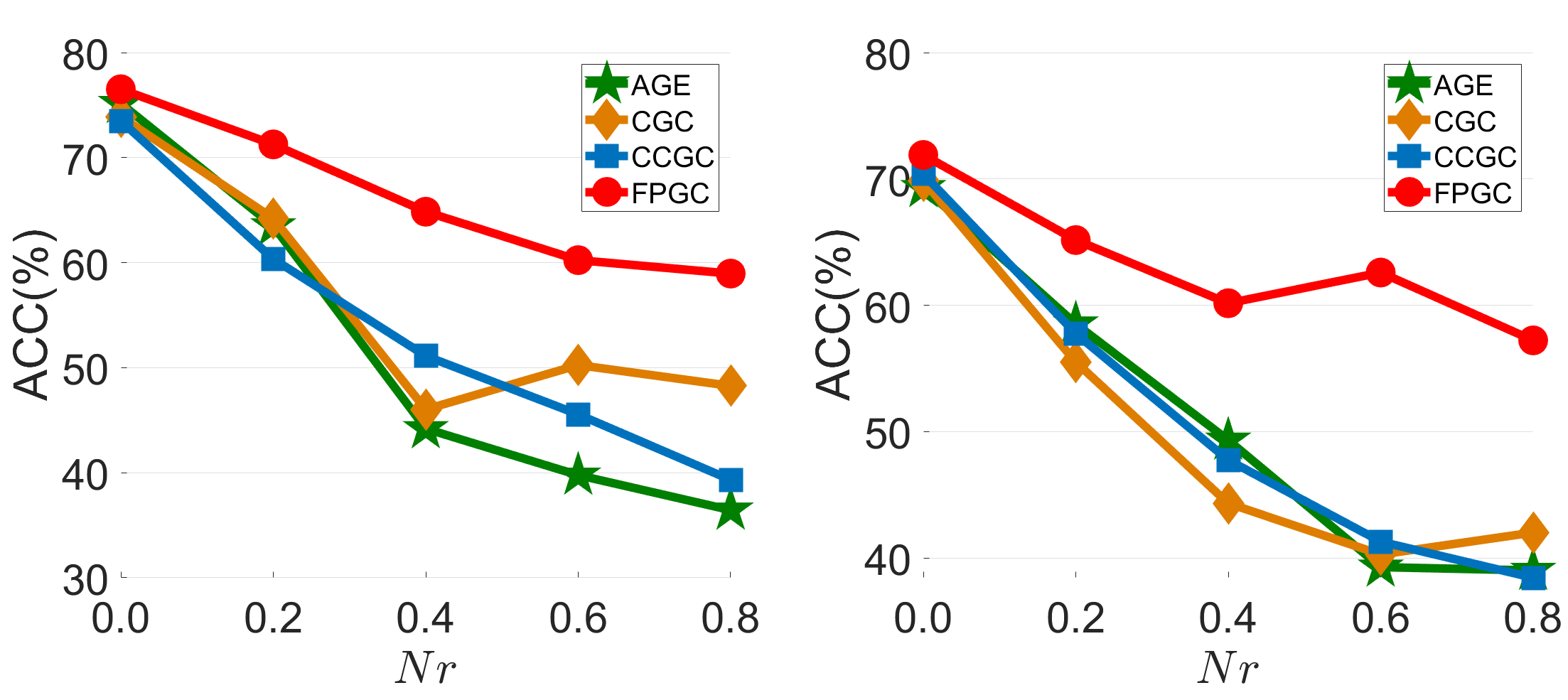}
    \caption{The robustness test on Cora (left) and Citeseer (right).}
    \label{rebnoise}
\end{figure}

\subsection{Ablation Study} \label{abla}
To see the impact of our proposed adaptive GNN, we compare PFGC with different combinations. The $l$-th layer aggregation becomes:

\begin{equation}
\begin{aligned}
H_1^{(l)} & = (1-\mu) \left(I - \frac{2}{3}\tilde{L}_M\right) H_1^{(l-1)} W^{(l-1)}\\
& +\mu \left(\frac{2}{3}\tilde{L}_G\right)H_1^{(l-1)} W^{(l-1)},\\
H_2^{(l)} & = (1-\mu) U_1 \operatorname{diag}\left(\left\{e^{1-\lambda_i^{(1)}}\right\}_{i=1}^N\right) U_1^{\top} H_2^{(l-1)} W^{(l-1)}\\
& +\mu U_2 \operatorname{diag}\left(\left\{e^{\lambda_i^{(2)}}\right\}_{i=1}^N\right) U_2^{\top} H_2^{(l-1)} W^{(l-1)},\\
H_3^{(l)} & = (1-\mu) \left(I - \frac{2}{3}\tilde{L}_M\right) H_3^{(l-1)} W^{(l-1)}\\
& +\mu U_2 \operatorname{diag}\left(\left\{e^{\lambda_i^{(2)}}\right\}_{i=1}^N\right) U_2^{\top} H_3^{(l-1)} W^{(l-1)},
\end{aligned}
\end{equation}
which are marked as PFGC-1, PFGC-2, and PFGC-3, respectively. The results are shown in Table \ref{abla1}. PFGC achieves the best performance in all cases. PFGC \textcolor{black}{outperforms} PFGC-1 by a lot, since PFGC-1 only considers traditional local filters. Traditional local filters are limited to \textcolor{black}{aggregating} information from immediate neighbors. They fail to capture the global homophilic connections inherent in $M$, where similar nodes may be topologically distant. Therefore, a global filter is \textcolor{black}{advantageous} in enhancing the quality of the representation. PFGC-3 achieves the worst performance in most cases, which is consistent with our theorem. PFGC-3 violates this spectral principle, \textcolor{black}{causing it to amplify} the noise in heterophilic graphs rather than the useful signal.

To further analyze the quality of our reconstructed graphs, we set the values in the graph as binary edges, which are marked as "PFGC-be". We also set the values in the lower triangular area to zero to make the graphs asymmetric, which is marked as "PFGC-as". Table \ref{abla1} shows that the results \textcolor{black}{decrease dramatically}. This demonstrates that both continuity and symmetry are crucial for the graph structure.

We also conduct experiments to evaluate the \textcolor{black}{efficiency} of our method. We compare PFGC with state-of-the-art methods in running time and GPU memory cost on Cora and Squirrel. The results are shown in Table \ref{time}. It can be seen that PFGC is fast and \textcolor{black}{requires minimal computing resources}. This is because our restructuring method only requires cosine similarity, and the loss functions are computationally efficient. Therefore, PFGC incurs no additional runtime compared to our baselines.

To validate the theoretical analysis, we apply the t-SNE technique \cite{van2008visualizing} to visualize the learned node features with different filters, that is, $h_1(L)$, $h_2(L)$, $h_3(L)$, and $h_4(L)$. The results are illustrated in Fig. \ref{tsne}. For a single filter, a global low-pass filter \textcolor{black}{improves} cluster discriminability more effectively than a local low-pass filter, whereas a local high-pass filter \textcolor{black}{performs better than} a global high-pass filter in the same context. This observation aligns with the theoretical analysis.

\begin{figure*}[t]
		\centering
  		\subfigure[Global low-pass filter $h_1(L)$.]{
			\includegraphics[width=0.23\linewidth]{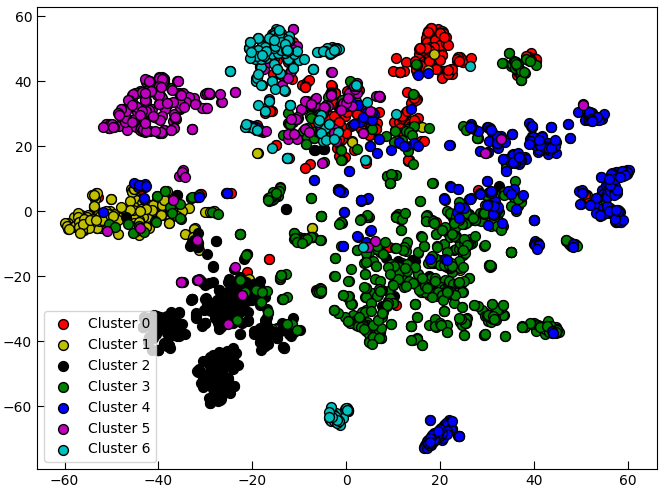}
		}
  		\subfigure[Local low-pass filter $h_2(L)$.]{
			\includegraphics[width=0.23\linewidth]{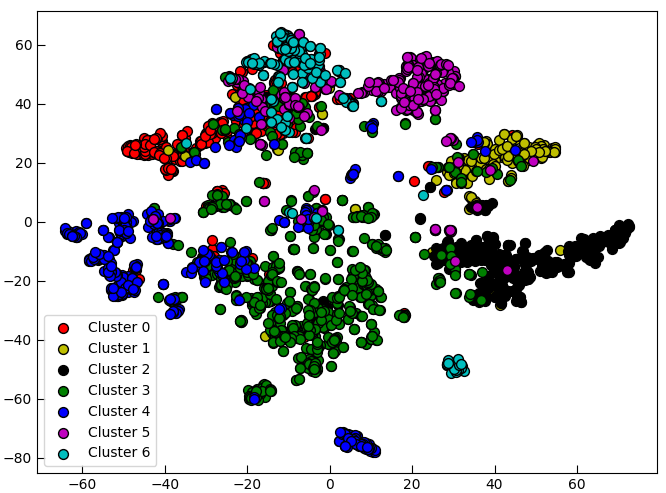}
		}
  		\subfigure[Global high-pass filter $h_3(L)$.]{
			\includegraphics[width=0.23\linewidth]{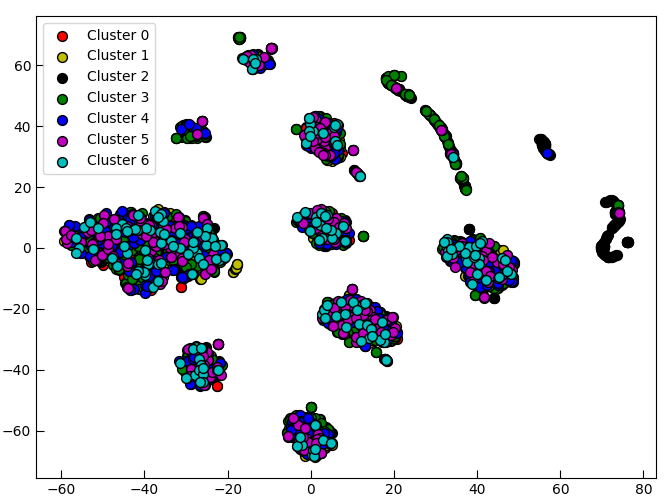}
		}
  		\subfigure[Local high-pass filter $h_4(L)$.]{
			\includegraphics[width=0.23\linewidth]{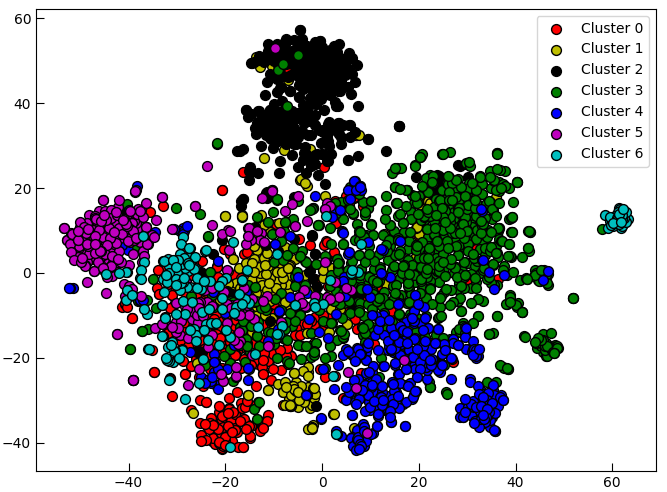}
		}
		\caption{t-SNE on Cora with different filters.}
		\label{tsne}
\end{figure*}

The homophily scores of the original graph and the constructed graphs are compared in Table \ref{homoresults}. It can be seen that the new graphs are highly homophilic or heterophilic. This shows that our graph restructuring method provides richer information.

\begin{table}[htbp]
		\centering
		\footnotesize
		\caption{The homophily scores of restructured graphs.}\label{homoresults}
		\setlength{\tabcolsep}{1mm}{\begin{tabular}{cccc}
    \toprule
          & A & M & G \\
    \midrule
    Cora  & 0.8100  & 0.8268$\uparrow$  & 0.1046\phantom{0}  \\
    Citeseer & 0.7355  & 0.7790$\uparrow$  & 0.0843\phantom{0}  \\
    Pubmed &0.8024 & 0.8204$\uparrow$ &0.2212\phantom{0} \\
    AMAP  & 0.8272 & 0.8836$\uparrow$ & 0.0653\phantom{0} \\
    UAT & 0.6978 & 0.7005$\uparrow$ & 0.1872\phantom{0} \\
    \midrule
    Cornell & 0.1227  & 0.4807$\uparrow$  & 0.2142\phantom{0}  \\
    Wisconsin & 0.1703  & 0.5059$\uparrow$  & 0.0360$\textcolor{green}{\downarrow}$ \\
    Washington  & 0.1530 & 0.5443${\uparrow}$ & 0.0653$\textcolor{green}{\downarrow}$\\

    Chameleon &0.2299 &0.5382${\uparrow}$ &0.1742$\textcolor{green}{\downarrow}$\\

    Squirrel & 0.2234  & 0.4781$\uparrow$  & 0.1999$\textcolor{green}{\downarrow}$\\

    Roman-Empire & 0.0469  & 0.5000$\uparrow$  & 0.0223$\textcolor{green}{\downarrow}$\\
    \bottomrule
    \end{tabular}}
	\end{table}%

\begin{figure}[!htbp]
    \centering
    \includegraphics[width=.8\linewidth]{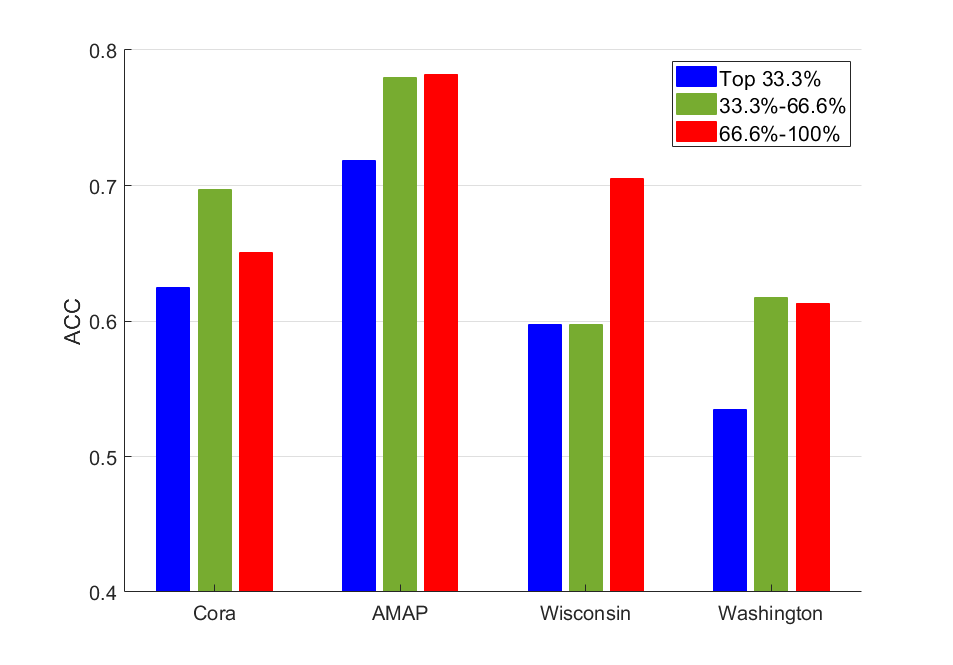}
    \caption{The results with masked features.}
    \label{sth04}
\end{figure}

To test the squeeze-and-excitation block's effect, we remove it and mark it as PFGC w/o SE. According to Table \ref{abla1}, there is a \textcolor{black}{performance drop}, indicating that the squeeze-and-excitation block improves the quality of the representation. We mask the final features $\tilde{H}$ according to the attention weights (i.e. $s$). We categorize it into three intervals: masking features of the top 33.3\% (\textcolor{black}{highest}) weights, masking features from 33.3\% to 66.6\% weights, and the remaining 66.6\% to 100\%. After masking, we use K-means on the features to obtain the ACC results. As depicted in Fig. \ref{sth04}, when masking the most important features (top 33.3\%), the performance \textcolor{black}{experiences the most significant decrease}. This suggests that the attention weights corresponding to the top features are crucial. For 33.3\%-66.6\% and 66.6\%-100\%, their effect \textcolor{black}{varies across} different datasets. Therefore, our proposed squeeze-and-excitation block can successfully \textcolor{black}{highlight} important features and thus enhance performance.


 \section{Experiments on Co-saliency Detection}
\subsection{Implementation Details}
In this section, we examine the performance of our method in co-saliency detection task, which aims to discover the common and salient foregrounds from a group of relevant images. We replace the GNN architecture in GCAGC \cite{zhang2020adaptive} with our proposed filter and only use features for graph restructuring. We add an additional baseline UFO \cite{su2023unified}, the recent SOTA technique for the co-saliency detection task. We also test on PFGC-1, PFGC-2, and PFGC-3 defined in \ref{abla}. The experimental settings and the training process follow GCAGC. We adopt the popular datasets iCoseg and Cosal2015. iCoseg contains 643 images in 38 groups, with common objects in each group sharing similar semantic traits despite pose and color variations; Cosal2015 includes 2,015 images in 50 categories, and each group faces challenges like complex environments. The evaluating metrics are Mean Absolute Error (MAE), F-measure score $F_\beta$, and S-measure score $S_m$.

\subsection{Results and Analysis}
The results are shown in Table \ref{rebtable}. It can be seen that PFGC \textcolor{black}{demonstrates improvement} in all cases. This is attributed to our filter, which captures local and global contextual information, \textcolor{black}{crucial} for identifying subtle co-salient features within diverse scenes. In addition, it is adept at handling both homophilic and heterophilic information, allowing it to identify common structures in images with varying degrees of similarity and background complexity. Note that the key difference between PFGC and GCAGC is the graph filter in the encoder. Thus, our graph filter has \textcolor{black}{considerable potential} for other visual tasks. \textcolor{black}{Moreover}, it can be observed that FPGC \textcolor{black}{outperforms} FPGC-1, FPGC-2, and FPGC-3. This result is consistent with the clustering outcomes and further validates the effectiveness of the proposed filter. PFGC-2 can perform better than PFGC-1 and PFGC-3 in most cases, once again highlighting the importance of the global homophily.

Fig. \ref{vis2} illustrates some visual comparison results. When the co-salient targets are surrounded by complex background clutter, strong semantic interference, and significant appearance fluctuations, our PFGC method \textcolor{black}{achieves superior} performance in terms of co-saliency compared to the baselines. This further highlights the distinct advantages of PFGC over GCAGC and UFO. In particular, UFO fails to identify the interior of an apple, while noise in the vicinity of the target is mistakenly recognized as a target item in both the boat and car images. Our method generates more accurate co-object masks \textcolor{black}{comparatively}, thereby validating its robustness and effectiveness.

We provide a visual comparison of different filters in Fig. \ref{vis}, which shows the encoder output. It can be seen that using only high-pass or low-pass filters leads to incomplete or failed saliency maps. Low-pass filters are most effective when applied globally, as they aggregate features across large regions, capturing semantic information such as object categories in image segmentation. By emphasizing low-frequency information, these filters smooth out noise and discontinuities at boundaries, enhancing the overall consistency of the output. However, excessive smoothing can \textcolor{black}{obscure} fine details, such as edges and textures, which are often critical for accurate representation. High-frequency components, on the other hand, preserve intricate details but are more prone to noise and instability when applied globally. To mitigate this, high-pass filters should be employed locally, focusing on fine-grained structural details while minimizing noise amplification. Striking the right balance between these frequency components is essential for achieving optimal performance across a variety of tasks.

\begin{table}[h]
\centering
\caption{Results of co-saliency detection. The best performance is marked in \textbf{bold}.}\label{rebtable}
\begin{tabular}{llcl|lll}
\hline 
\multirow{2}{*}{Methods} & \multicolumn{3}{c|}{$\operatorname{iCoseg}$} & \multicolumn{3}{c}{$\operatorname{Cosal2015}$} \\
\cline{2-7}
& MAE $\downarrow$ & $F_\beta$ $\uparrow$& $S_m$ $\uparrow$ & MAE $\downarrow$ & $F_\beta$ $\uparrow$& $S_m$ $\uparrow$ \\
\hline 
GCAGC & 0.076 & 0.853 & 0.821 & 0.089 & 0.843 & 0.822 \\
UFO &\phantom{0}\phantom{0}-&-&\phantom{0}\phantom{0}- &0.064 &0.865 &0.860 \\
PFGC-1 &0.065&0.811&0.801&0.093&0.825& 0.814 \\
PFGC-2 &0.060&0.855&0.825&0.060&0.844& 0.852\\
PFGC-3 &0.082&0.843&0.817&0.071&0.817& 0.817\\
PFGC & \textbf{0.053} & \textbf{0.861} & \textbf{0.832} & \textbf{0.056} & \textbf{0.886} & \textbf{0.865} \\
\hline
\end{tabular}
\end{table}

\begin{figure}[htbp]
    \centering
    \textbf{iCoseg}

    \parbox[b]{0.01\textwidth}{\centering \scriptsize \textbf{Input}}
    \hfill
    \begin{minipage}{0.07\textwidth}\centering
        \includegraphics[width=\textwidth]{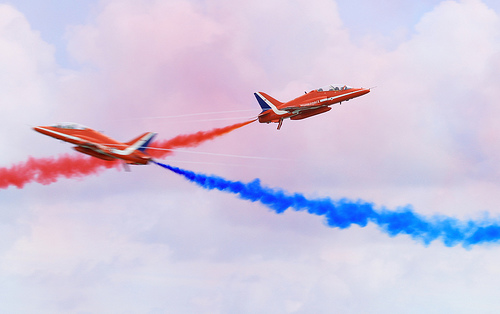}
    \end{minipage}\hfill
    \begin{minipage}{0.07\textwidth}\centering
        \includegraphics[width=\textwidth]{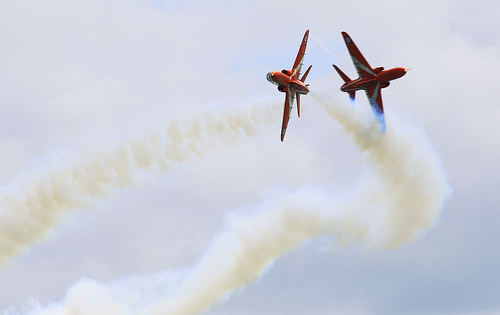}
    \end{minipage}\hfill
    \begin{minipage}{0.07\textwidth}\centering
        \includegraphics[width=\textwidth]{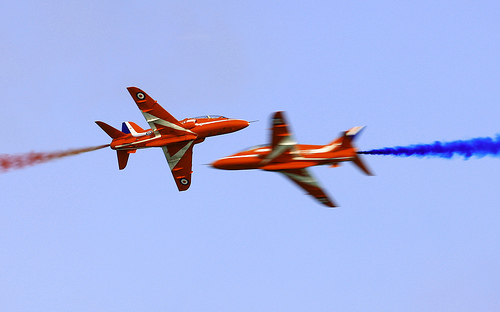}
    \end{minipage}\hfill
    \begin{minipage}{0.07\textwidth}\centering
        \includegraphics[height=0.9\textwidth]{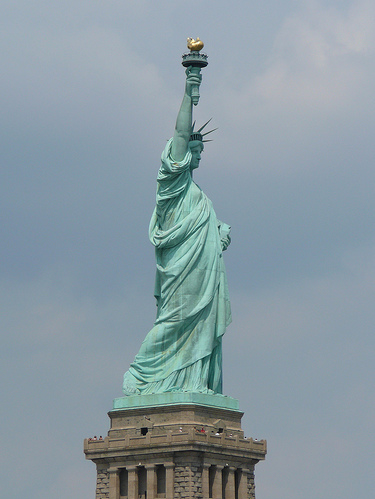}
    \end{minipage}\hfill
    \begin{minipage}{0.07\textwidth}\centering
        \includegraphics[height=0.8\textwidth]{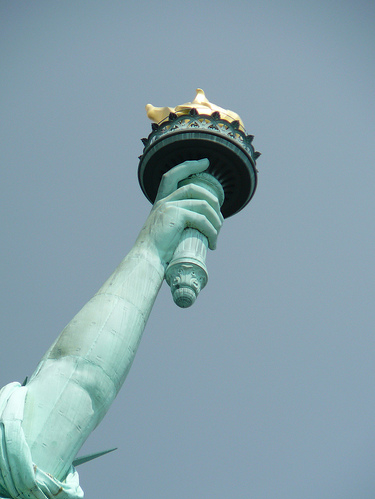}
    \end{minipage}

    \vspace{0.1em}

    \parbox[b]{0.01\textwidth}{\centering \scriptsize \textbf{GT}}
    \hfill
    \begin{minipage}{0.07\textwidth}\centering
        \includegraphics[width=\textwidth]{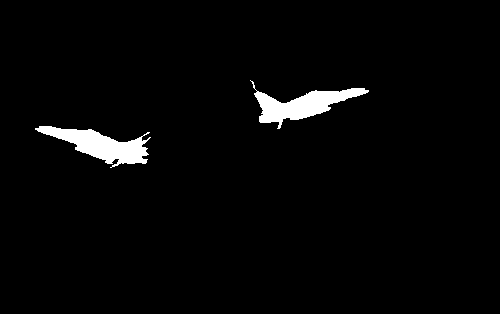}
    \end{minipage}\hfill
    \begin{minipage}{0.07\textwidth}\centering
        \includegraphics[width=\textwidth]{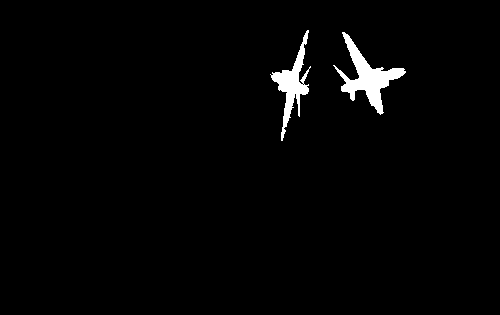}
    \end{minipage}\hfill
    \begin{minipage}{0.07\textwidth}\centering
        \includegraphics[width=\textwidth]{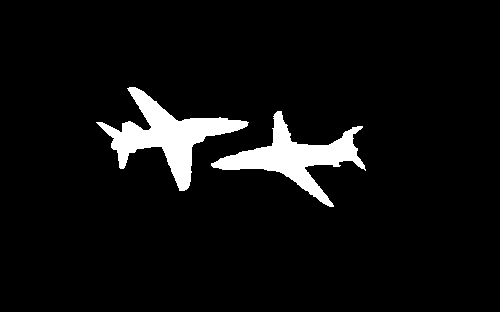}
    \end{minipage}\hfill
    \begin{minipage}{0.07\textwidth}\centering
        \includegraphics[height=0.9\textwidth]{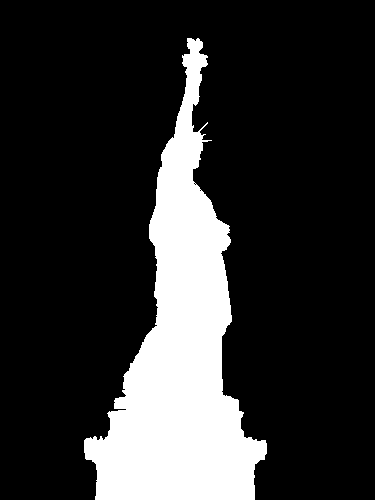}
    \end{minipage}\hfill
    \begin{minipage}{0.07\textwidth}\centering
        \includegraphics[height=0.8\textwidth]{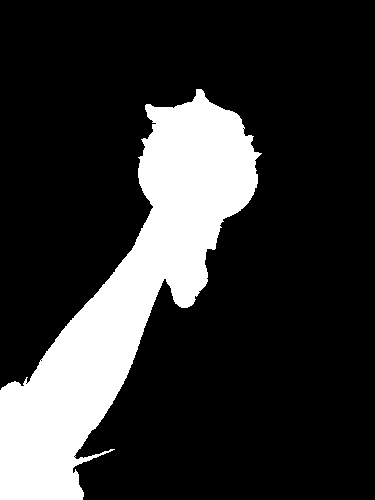}
    \end{minipage}

    \vspace{0.1em}

    \parbox[b]{0.01\textwidth}{\centering \scriptsize \textcolor[rgb]{1,0,0}{\textbf{PFGC}}}
    \hfill
    \begin{minipage}{0.07\textwidth}\centering
        \includegraphics[width=\textwidth]{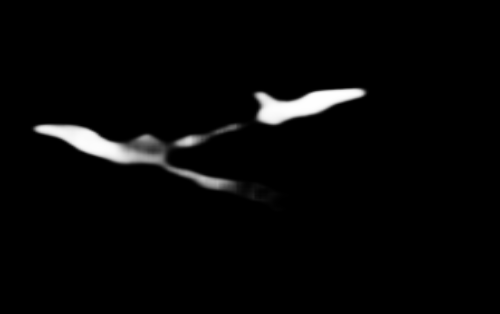}
    \end{minipage}\hfill
    \begin{minipage}{0.07\textwidth}\centering
        \includegraphics[width=\textwidth]{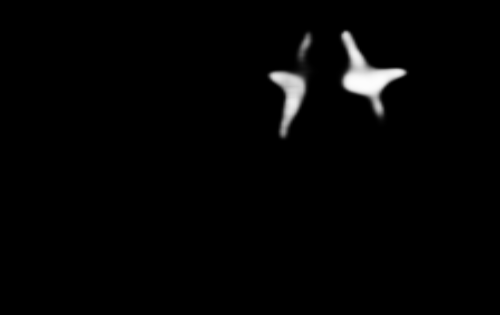}
    \end{minipage}\hfill
    \begin{minipage}{0.07\textwidth}\centering
        \includegraphics[width=\textwidth]{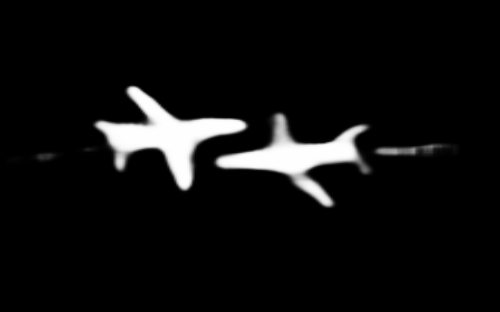}
    \end{minipage}\hfill
    \begin{minipage}{0.07\textwidth}\centering
        \includegraphics[height=0.9\textwidth]{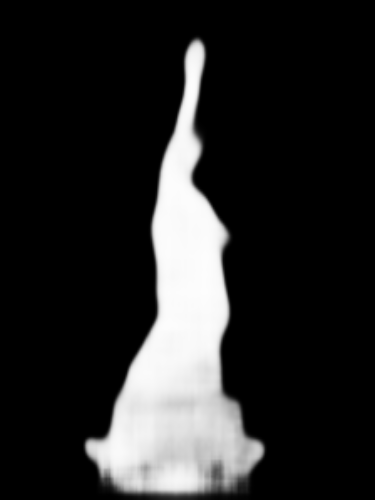}
    \end{minipage}\hfill
    \begin{minipage}{0.07\textwidth}\centering
        \includegraphics[height=0.8\textwidth]{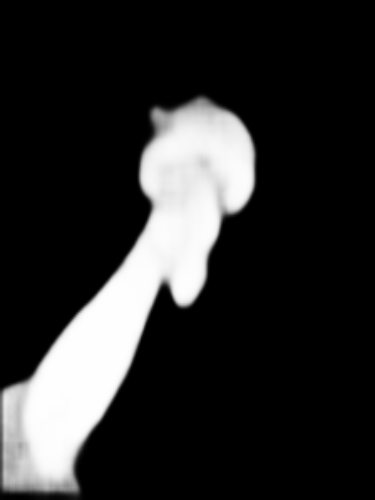}
    \end{minipage}

    \vspace{0.1em}

    \parbox[b]{0.01\textwidth}{\centering \tiny \textbf{GCAGC}}
    \hfill
    \begin{minipage}{0.07\textwidth}\centering
        \includegraphics[width=\textwidth]{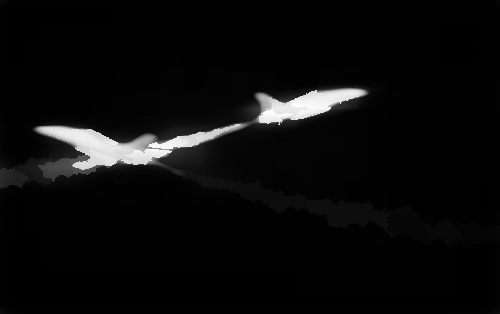}
    \end{minipage}\hfill
    \begin{minipage}{0.07\textwidth}\centering
        \includegraphics[width=\textwidth]{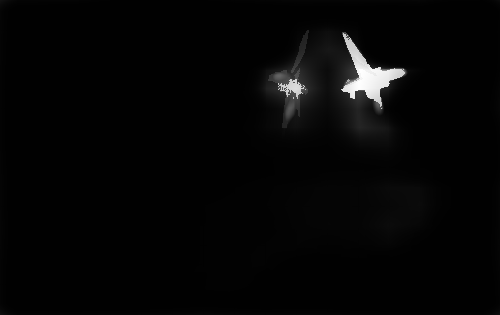}
    \end{minipage}\hfill
    \begin{minipage}{0.07\textwidth}\centering
        \includegraphics[width=\textwidth]{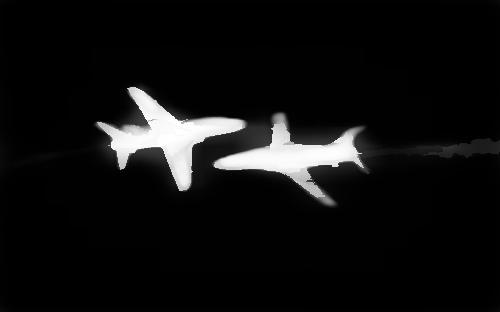}
    \end{minipage}\hfill
    \begin{minipage}{0.07\textwidth}\centering
        \includegraphics[height=0.9\textwidth]{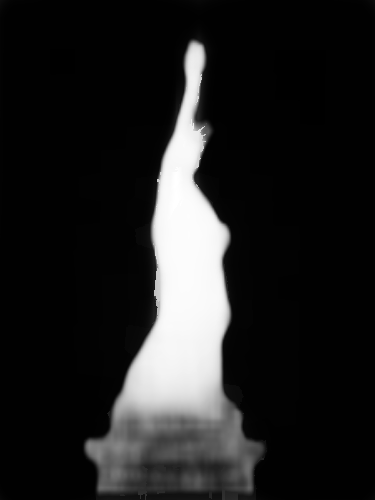}
    \end{minipage}\hfill
    \begin{minipage}{0.07\textwidth}\centering
        \includegraphics[height=0.8\textwidth]{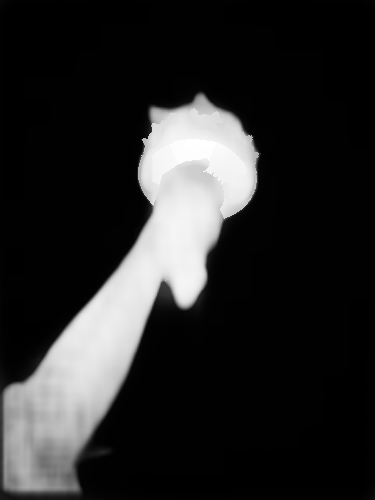}
    \end{minipage}

    \vspace{0.3em}
    \textbf{Cosal2015}

    \parbox[b]{0.01\textwidth}{\centering \scriptsize \textbf{Input}}
    \hfill
    \begin{minipage}{0.07\textwidth}\centering
        \includegraphics[width=\textwidth]{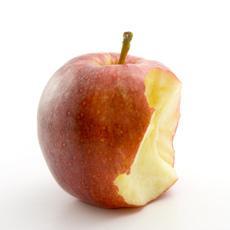}
    \end{minipage}\hfill
    \begin{minipage}{0.07\textwidth}\centering
        \includegraphics[width=\textwidth]{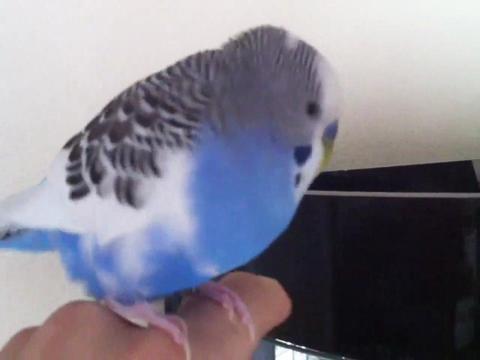}
    \end{minipage}\hfill
    \begin{minipage}{0.07\textwidth}\centering
        \includegraphics[width=\textwidth]{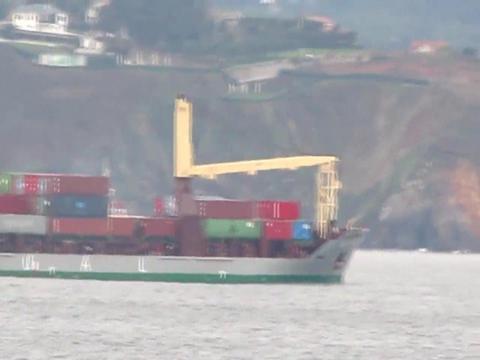}
    \end{minipage}\hfill
    \begin{minipage}{0.07\textwidth}\centering
        \includegraphics[width=\textwidth]{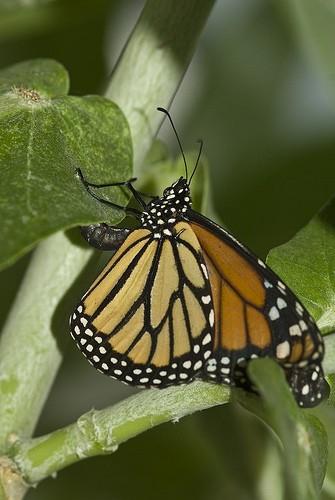}
    \end{minipage}\hfill
    \begin{minipage}{0.07\textwidth}\centering
        \includegraphics[width=\textwidth]{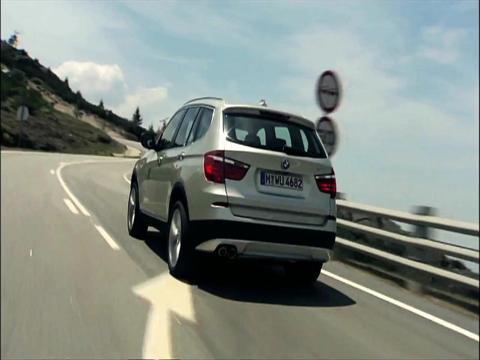}
    \end{minipage}

    \vspace{0.1em}

    \parbox[b]{0.01\textwidth}{\centering \scriptsize \textbf{GT}}
    \hfill
    \begin{minipage}{0.07\textwidth}\centering
        \includegraphics[width=\textwidth]{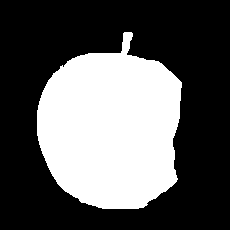}
    \end{minipage}\hfill
    \begin{minipage}{0.07\textwidth}\centering
        \includegraphics[width=\textwidth]{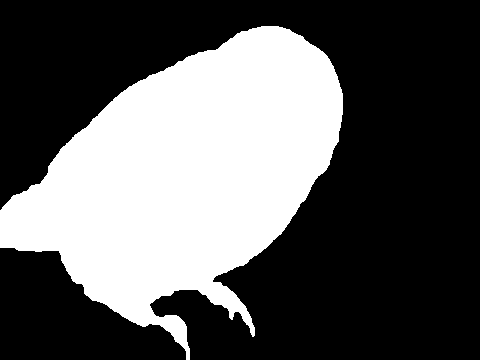}
    \end{minipage}\hfill
    \begin{minipage}{0.07\textwidth}\centering
        \includegraphics[width=\textwidth]{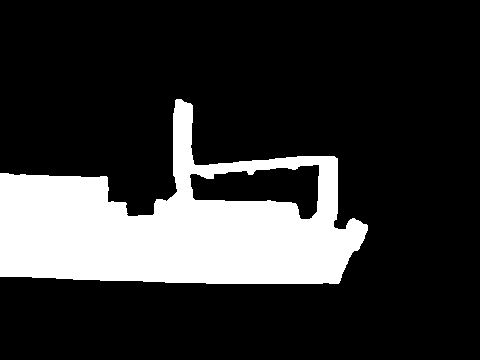}
    \end{minipage}\hfill
    \begin{minipage}{0.07\textwidth}\centering
        \includegraphics[width=\textwidth]{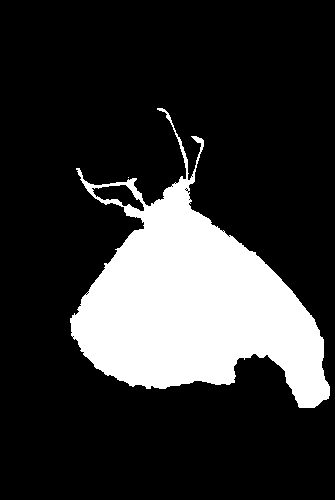}
    \end{minipage}\hfill
    \begin{minipage}{0.07\textwidth}\centering
        \includegraphics[width=\textwidth]{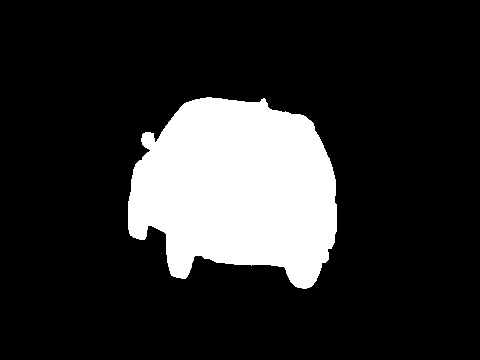}
    \end{minipage}

    \vspace{0.1em}

    \parbox[b]{0.01\textwidth}{\centering \scriptsize \textcolor[rgb]{1,0,0}{\textbf{PFGC}}}
    \hfill
    \begin{minipage}{0.07\textwidth}\centering
        \includegraphics[width=\textwidth]{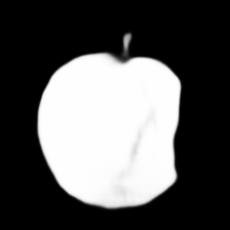}
    \end{minipage}\hfill
    \begin{minipage}{0.07\textwidth}\centering
        \includegraphics[width=\textwidth]{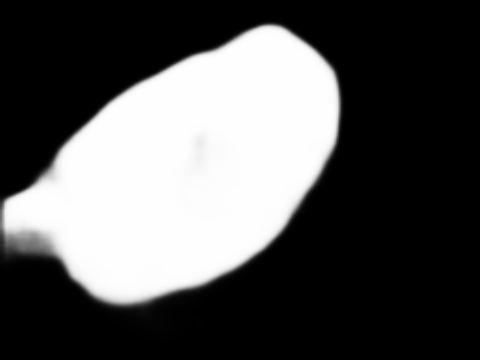}
    \end{minipage}\hfill
    \begin{minipage}{0.07\textwidth}\centering
        \includegraphics[width=\textwidth]{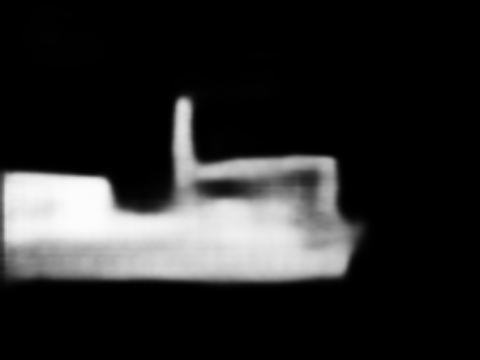}
    \end{minipage}\hfill
    \begin{minipage}{0.07\textwidth}\centering
        \includegraphics[width=\textwidth]{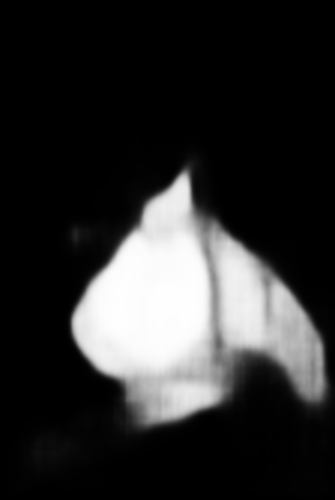}
    \end{minipage}\hfill
    \begin{minipage}{0.07\textwidth}\centering
        \includegraphics[width=\textwidth]{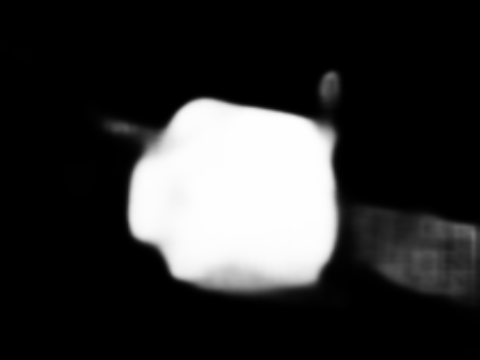}
    \end{minipage}

    \vspace{0.1em}

    \parbox[b]{0.01\textwidth}{\centering \tiny \textbf{GCAGC}}
    \hfill
    \begin{minipage}{0.07\textwidth}\centering
        \includegraphics[width=\textwidth]{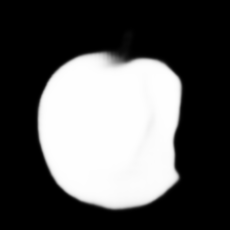}
    \end{minipage}\hfill
    \begin{minipage}{0.07\textwidth}\centering
        \includegraphics[width=\textwidth]{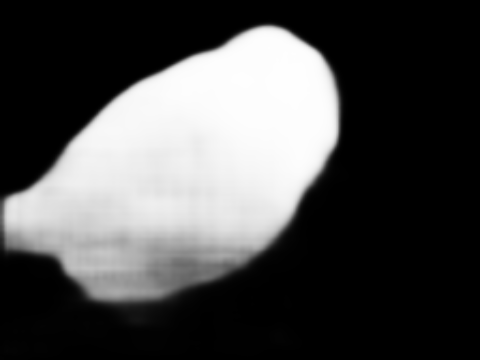}
    \end{minipage}\hfill
    \begin{minipage}{0.07\textwidth}\centering
        \includegraphics[width=\textwidth]{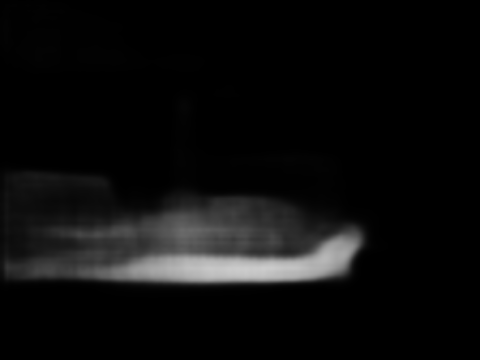}
    \end{minipage}\hfill
    \begin{minipage}{0.07\textwidth}\centering
        \includegraphics[width=\textwidth]{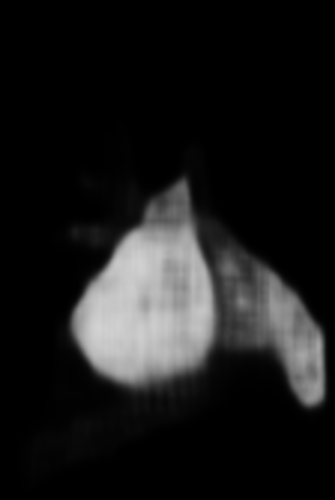}
    \end{minipage}\hfill
    \begin{minipage}{0.07\textwidth}\centering
        \includegraphics[width=\textwidth]{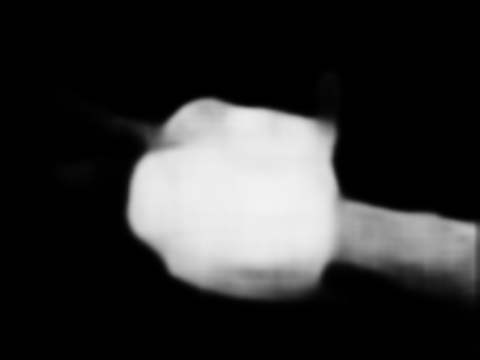}
    \end{minipage}

    \vspace{0.1em}

    \parbox[b]{0.01\textwidth}{\centering \tiny \textbf{UFO}}
    \hfill
    \begin{minipage}{0.07\textwidth}\centering
        \includegraphics[width=\textwidth]{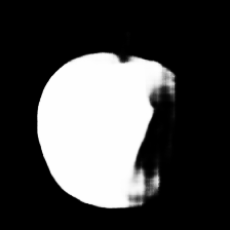}
    \end{minipage}\hfill
    \begin{minipage}{0.07\textwidth}\centering
        \includegraphics[width=\textwidth]{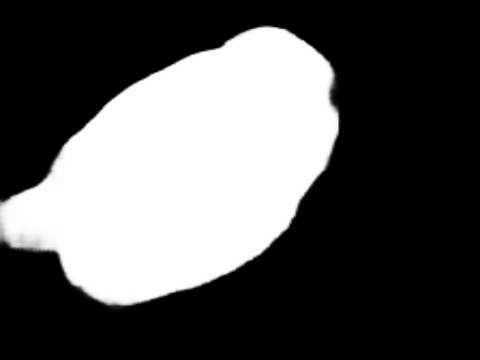}
    \end{minipage}\hfill
    \begin{minipage}{0.07\textwidth}\centering
        \includegraphics[width=\textwidth]{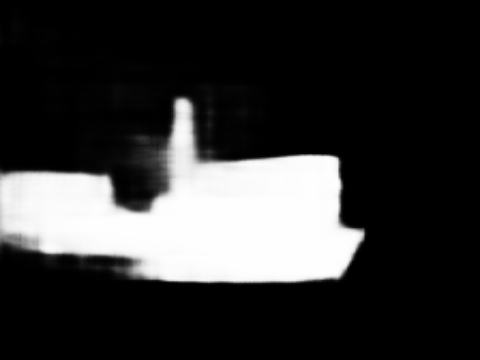}
    \end{minipage}\hfill
    \begin{minipage}{0.07\textwidth}\centering
        \includegraphics[width=\textwidth]{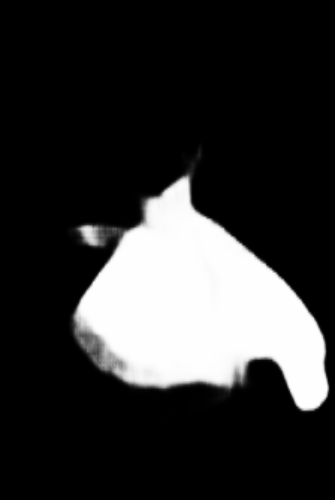}
    \end{minipage}\hfill
    \begin{minipage}{0.07\textwidth}\centering
        \includegraphics[width=\textwidth]{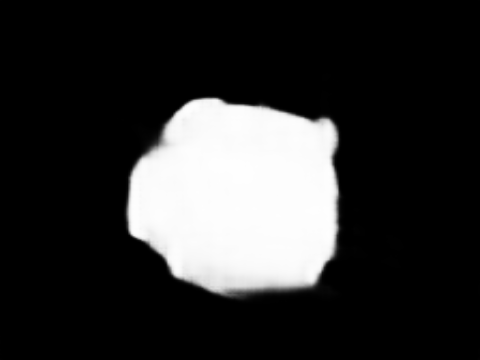}
    \end{minipage}

    \caption{Visual comparisons of our PFGC method against GCAGC and UFO. ``GT'' indicates the ground truth.}
    \label{vis2}
\end{figure}

\begin{figure}[htbp]
    \centering
    \parbox[b]{0.01\textwidth}{\centering \scriptsize \textbf{Image Group}}
    \hfill
    \begin{minipage}{0.07\textwidth}\centering
        \includegraphics[width=\textwidth]{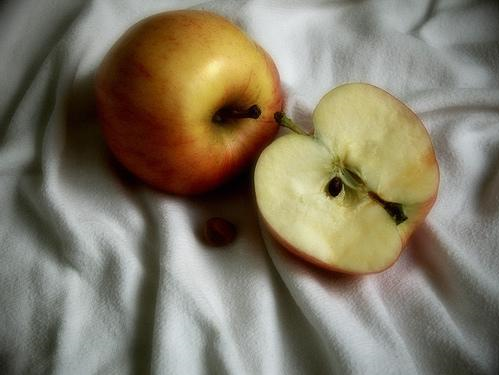}
    \end{minipage}\hfill
    \begin{minipage}{0.07\textwidth}\centering
        \includegraphics[width=\textwidth]{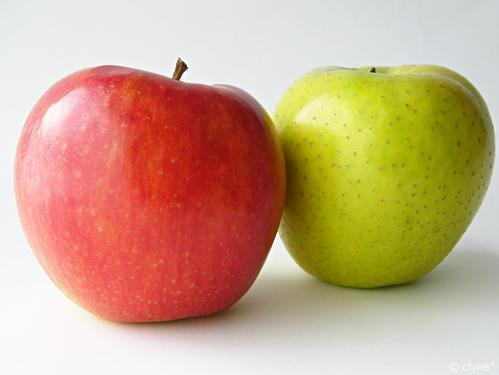}
    \end{minipage}\hfill
    \begin{minipage}{0.07\textwidth}\centering
        \includegraphics[width=\textwidth]{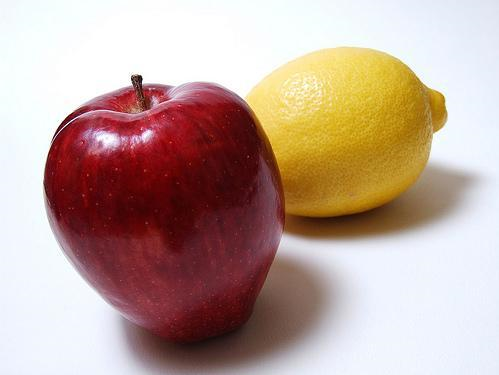}
    \end{minipage}\hfill
    \begin{minipage}{0.07\textwidth}\centering
        \includegraphics[width=\textwidth]{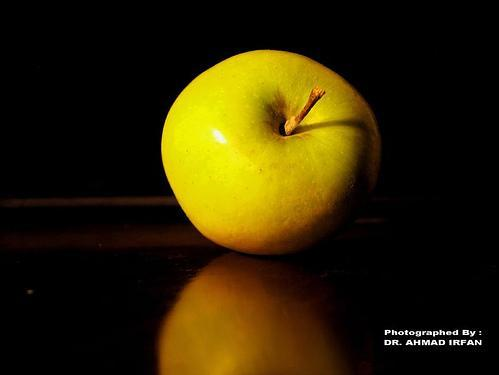}
    \end{minipage}\hfill
    \begin{minipage}{0.07\textwidth}\centering
        \includegraphics[height=0.8\textwidth]{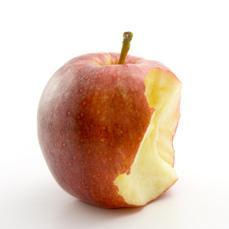}
    \end{minipage}

    \vspace{0.1em}

    \parbox[b]{0.01\textwidth}{\centering \scriptsize \textbf{$h_1(L)$}}
    \hfill
    \begin{minipage}{0.07\textwidth}\centering
        \includegraphics[width=\textwidth]{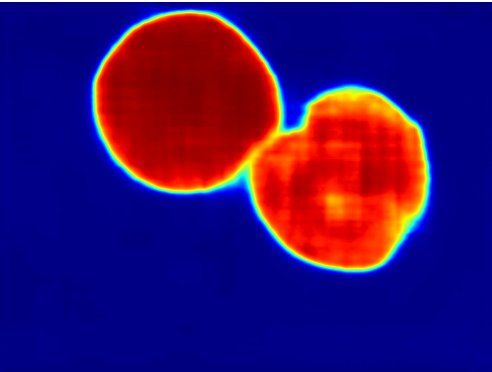}
    \end{minipage}\hfill
    \begin{minipage}{0.07\textwidth}\centering
        \includegraphics[width=\textwidth]{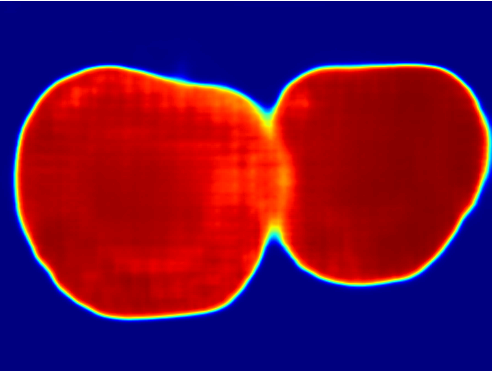}
    \end{minipage}\hfill
    \begin{minipage}{0.07\textwidth}\centering
        \includegraphics[width=\textwidth]{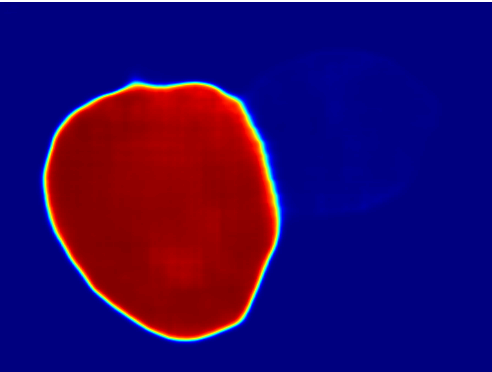}
    \end{minipage}\hfill
    \begin{minipage}{0.07\textwidth}\centering
        \includegraphics[width=\textwidth]{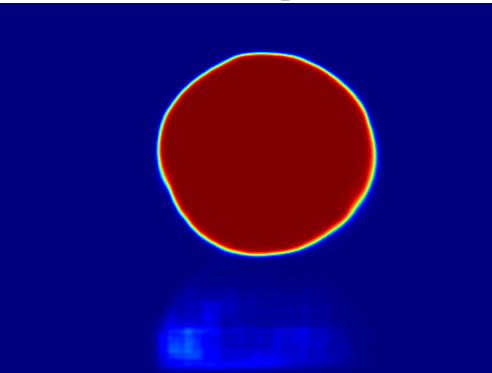}
    \end{minipage}\hfill
    \begin{minipage}{0.07\textwidth}\centering
        \includegraphics[height=0.8\textwidth]{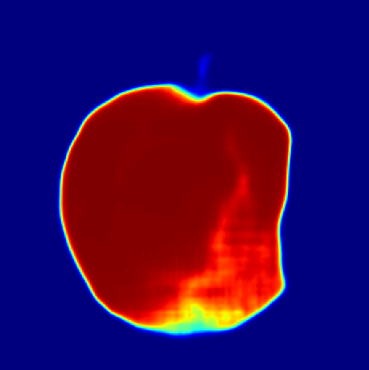}
    \end{minipage}

    \vspace{0.1em}

    \parbox[b]{0.01\textwidth}{\centering \scriptsize \textbf{$h_2(L)$}}
    \hfill
    \begin{minipage}{0.07\textwidth}\centering
        \includegraphics[width=\textwidth]{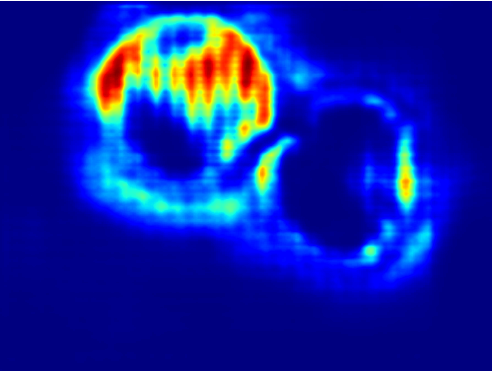}
    \end{minipage}\hfill
    \begin{minipage}{0.07\textwidth}\centering
        \includegraphics[width=\textwidth]{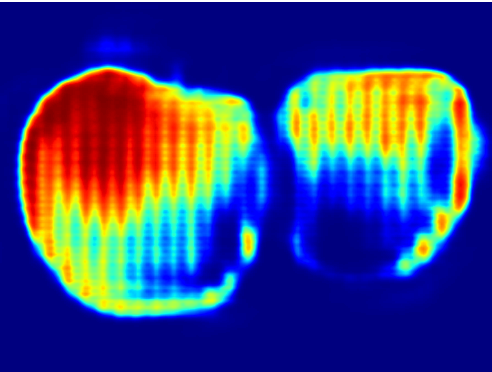}
    \end{minipage}\hfill
    \begin{minipage}{0.07\textwidth}\centering
        \includegraphics[width=\textwidth]{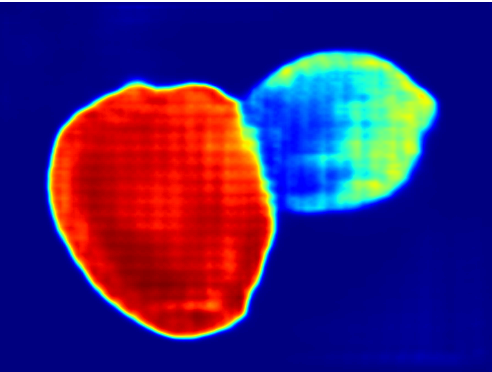}
    \end{minipage}\hfill
    \begin{minipage}{0.07\textwidth}\centering
        \includegraphics[width=\textwidth]{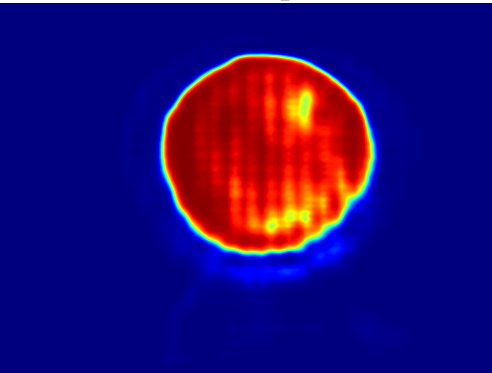}
    \end{minipage}\hfill
    \begin{minipage}{0.07\textwidth}\centering
        \includegraphics[height=0.8\textwidth]{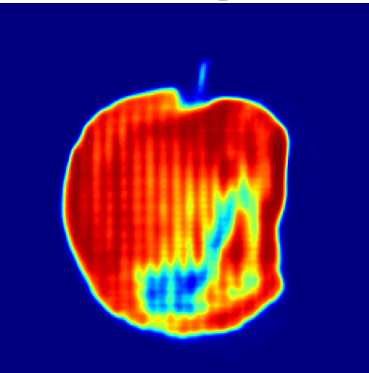}
    \end{minipage}

    \vspace{0.1em}

    \parbox[b]{0.01\textwidth}{\centering \scriptsize \textbf{$h_3(L)$}}
    \hfill
    \begin{minipage}{0.07\textwidth}\centering
        \includegraphics[width=\textwidth]{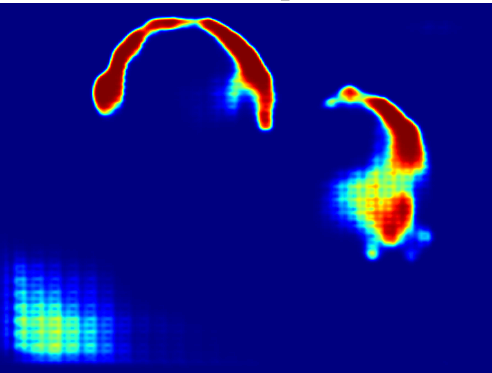}
    \end{minipage}\hfill
    \begin{minipage}{0.07\textwidth}\centering
        \includegraphics[width=\textwidth]{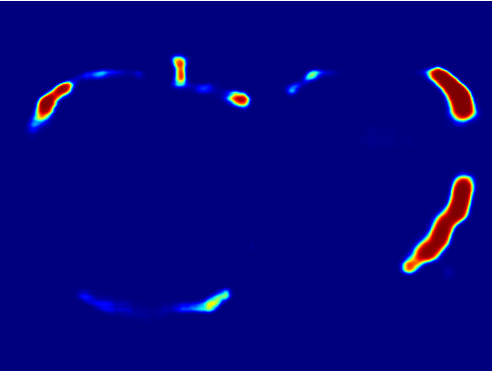}
    \end{minipage}\hfill
    \begin{minipage}{0.07\textwidth}\centering
        \includegraphics[width=\textwidth]{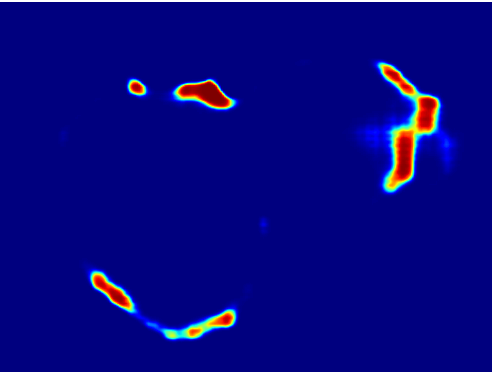}
    \end{minipage}\hfill
    \begin{minipage}{0.07\textwidth}\centering
        \includegraphics[width=\textwidth]{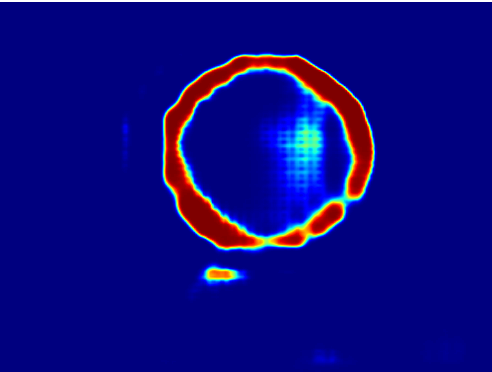}
    \end{minipage}\hfill
    \begin{minipage}{0.07\textwidth}\centering
        \includegraphics[height=0.8\textwidth]{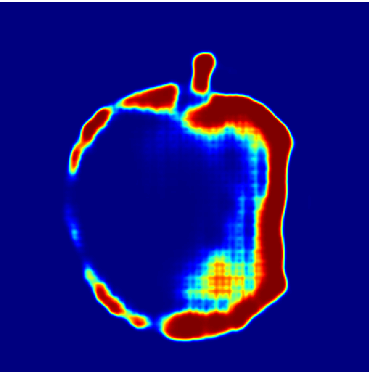}
    \end{minipage}

    \vspace{0.1em}

    \parbox[b]{0.01\textwidth}{\centering \scriptsize \textbf{$h_4(L)$}}
    \hfill
    \begin{minipage}{0.07\textwidth}\centering
        \includegraphics[width=\textwidth]{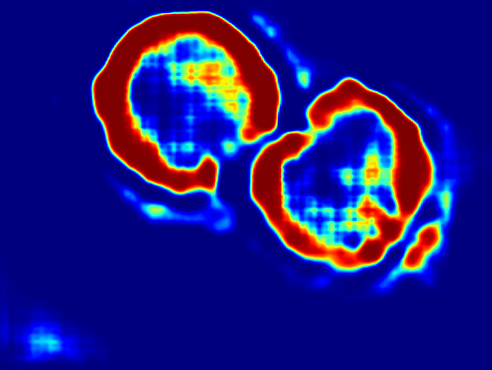}
    \end{minipage}\hfill
    \begin{minipage}{0.07\textwidth}\centering
        \includegraphics[width=\textwidth]{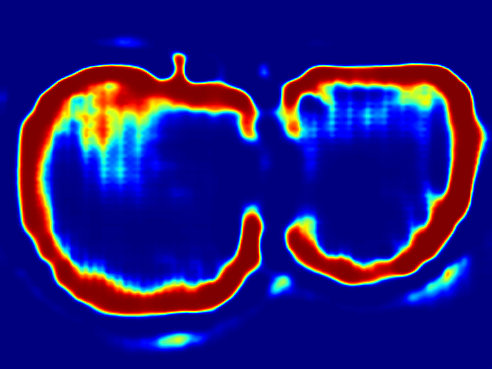}
    \end{minipage}\hfill
    \begin{minipage}{0.07\textwidth}\centering
        \includegraphics[width=\textwidth]{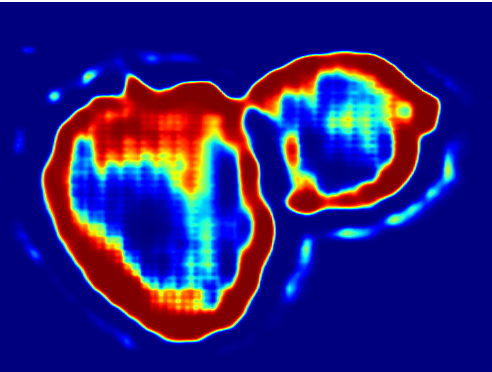}
    \end{minipage}\hfill
    \begin{minipage}{0.07\textwidth}\centering
        \includegraphics[width=\textwidth]{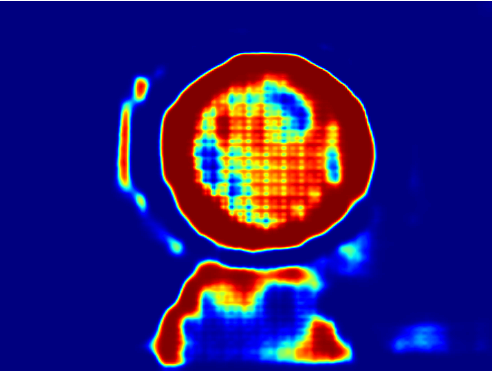}
    \end{minipage}\hfill
    \begin{minipage}{0.07\textwidth}\centering
        \includegraphics[height=0.8\textwidth]{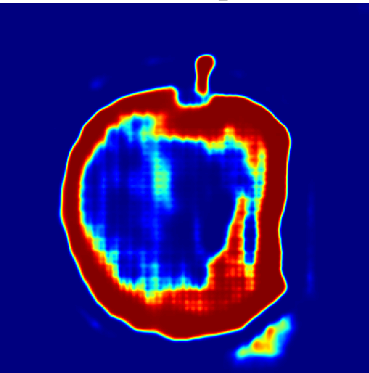}
    \end{minipage}

    \vspace{0.1em}

    \parbox[b]{0.01\textwidth}{\centering \scriptsize \textcolor[rgb]{1,0,0}{\textbf{PFGC}}}
    \hfill
    \begin{minipage}{0.07\textwidth}\centering
        \includegraphics[width=\textwidth]{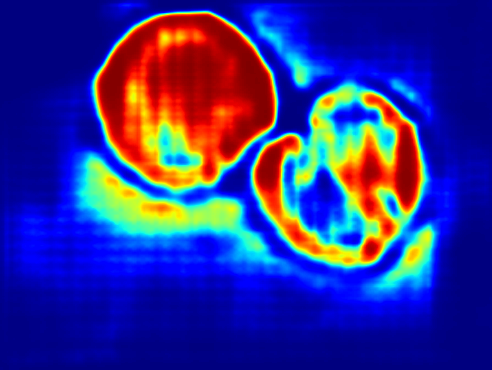}
    \end{minipage}\hfill
    \begin{minipage}{0.07\textwidth}\centering
        \includegraphics[width=\textwidth]{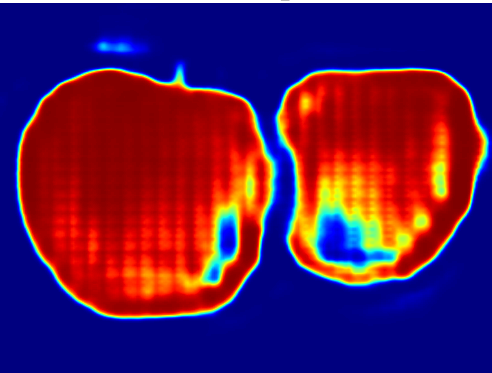}
    \end{minipage}\hfill
    \begin{minipage}{0.07\textwidth}\centering
        \includegraphics[width=\textwidth]{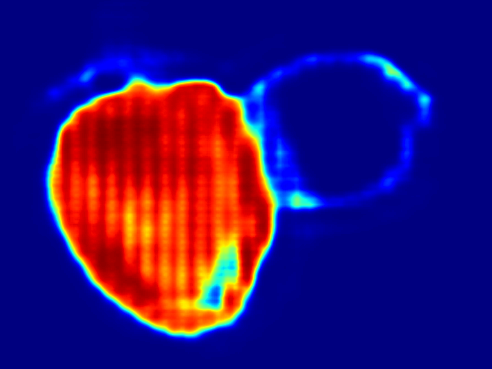}
    \end{minipage}\hfill
    \begin{minipage}{0.07\textwidth}\centering
        \includegraphics[width=\textwidth]{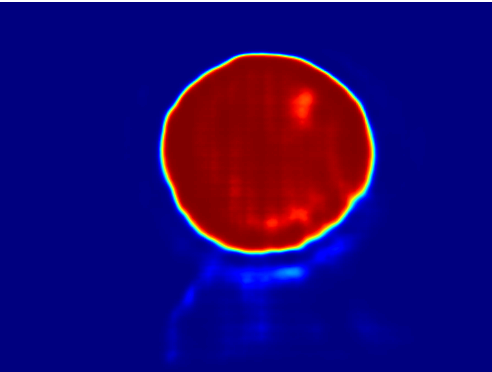}
    \end{minipage}\hfill
    \begin{minipage}{0.07\textwidth}\centering
        \includegraphics[height=0.8\textwidth]{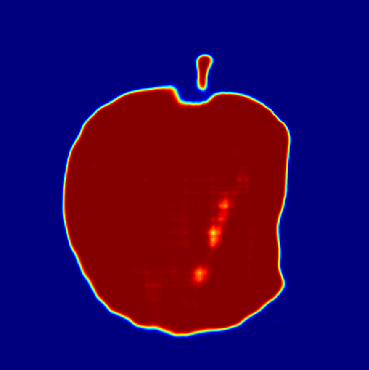}
    \end{minipage}

    \caption{Different filters' visualization results.}
    \label{vis}
\end{figure}
 
 \section{Conclusion}
This study presents an \textcolor{black}{innovative} approach to graph clustering, specifically addressing the challenges posed by homophilic and heterophilic graph structures. One of our key contributions is the development of unsupervised strategies for graph restructuring, enabling the \textcolor{black}{extraction} of homophilic and heterophilic information based on the commonality of graphs. In addition, we design an adaptive GNN to better \textcolor{black}{leverage} the characteristics of restructured graphs. The incorporation of the squeeze-and-excitation block \textcolor{black}{enhances} the significance of crucial features to further improve performance, marking the \textcolor{black}{first use} of this concept in graph clustering. Theoretical and empirical results validate the effectiveness of our approach in both graph and visual tasks. \textcolor{black}{Future work can explore more efficient sampling strategies, adaptive model compression, and broader domain evaluations to further enhance the applicability and robustness of the proposed method.}
 
\bibliography{sth}
\bibliographystyle{IEEEtran}

\begin{IEEEbiography}[{\includegraphics[width=1in,height=1.25in,clip,keepaspectratio]{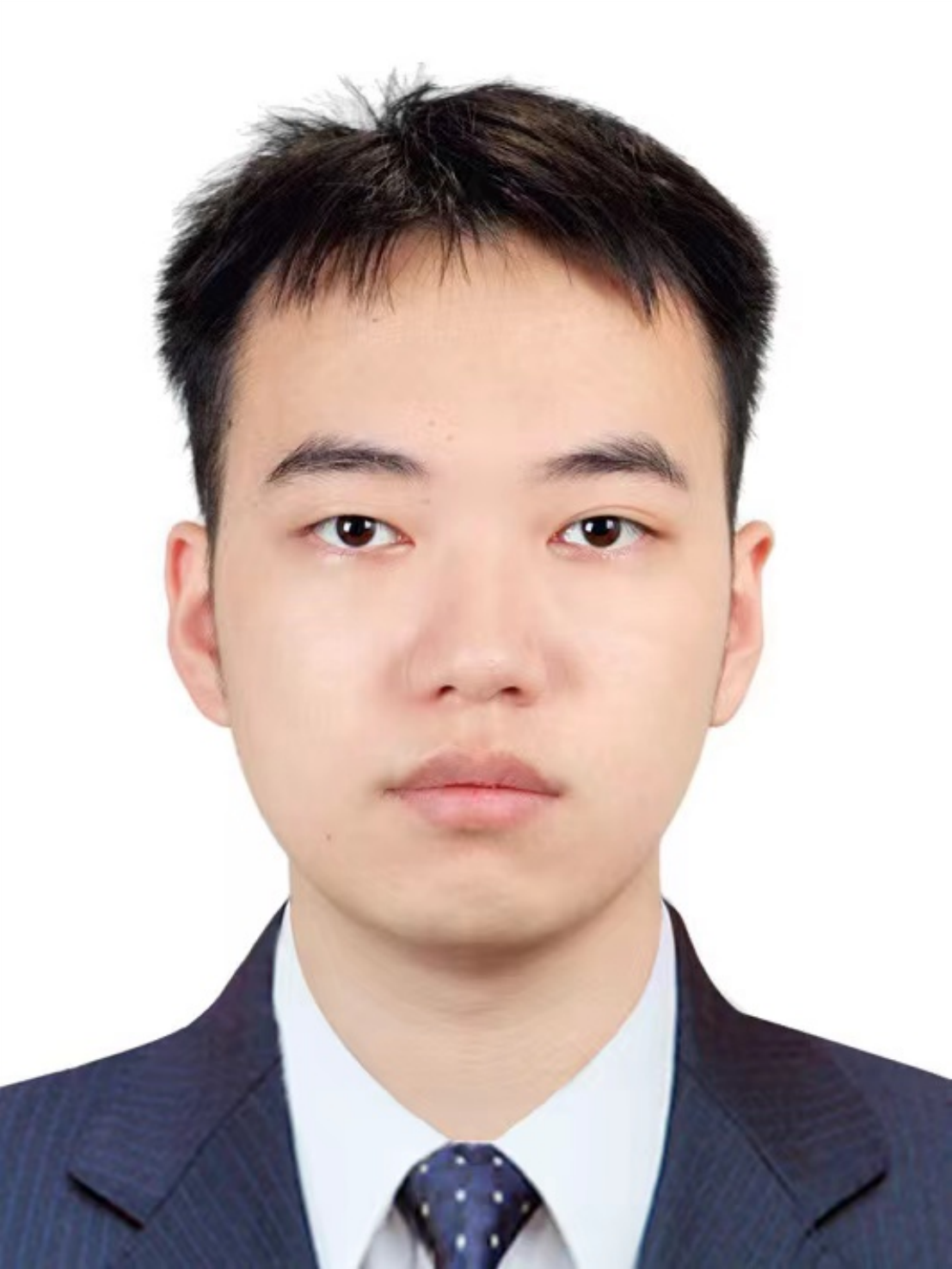}}]{Xuanting Xie}
received the B.Sc. degree in Software Engineering in
2021 and is currently studying for a Ph.D. degree with School of Computer
Science and Engineering, the University of Electronic Science and Technology
of China, Chengdu, China. His main research interests include graph clustering, graph foundation models, and LLM.
\end{IEEEbiography}

\begin{IEEEbiography}[{\includegraphics[width=1in,height=1.25in,clip,keepaspectratio]{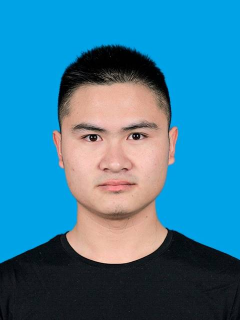}}]{Erlin Pan}
received the B.Sc. degree in computer science and technology in 2021 and the M.S. degree in 2024 from the School of Computer Science and Engineering, University of Electronic Science and Technology of China, Chengdu, China. He is currently working at Alibaba Group. His main research interests include graph learning, multiview learning, and clustering.
\end{IEEEbiography}
\begin{IEEEbiography}[{\includegraphics[width=1in,height=1.25in,clip,keepaspectratio]{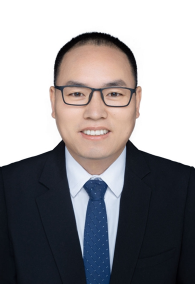}}]{Zhao Kang}
received the Ph.D. degree in computer science from the Southern Illinois University Carbondale,
Carbondale, IL, USA, in 2017. He is currently a Professor in
the School of Computer Science and Engineering, University
of Electronic Science and Technology of China, Chengdu,
China. He has published more than 100 research articles in
top-tier conferences and journals, including ICML, NeurIPS,
ICLR, AAAI, IJCAI, IEEE TRANSACTIONS
ON CYBERNETICS, IEEE TRANSACTIONS ON IMAGE
PROCESSING, IEEE TRANSACTIONS ON KNOWLEDGE AND
DATA ENGINEERING, and IEEE TRANSACTIONS ON NEURAL
NETWORKS AND LEARNING SYSTEMS. His research interests include graph machine
learning and large language models. Dr. Kang has been an AC/SPC/PC Member for a number of top conferences, such as NeurIPS, ICML, ICLR, AAAI, IJCAI, CVPR, and SIGKDD. He serves as an Associate Editor for Neural Networks and Pattern Recognition.
\end{IEEEbiography}
\begin{IEEEbiography}[{\includegraphics[width=1in,height=1.25in,clip,keepaspectratio]{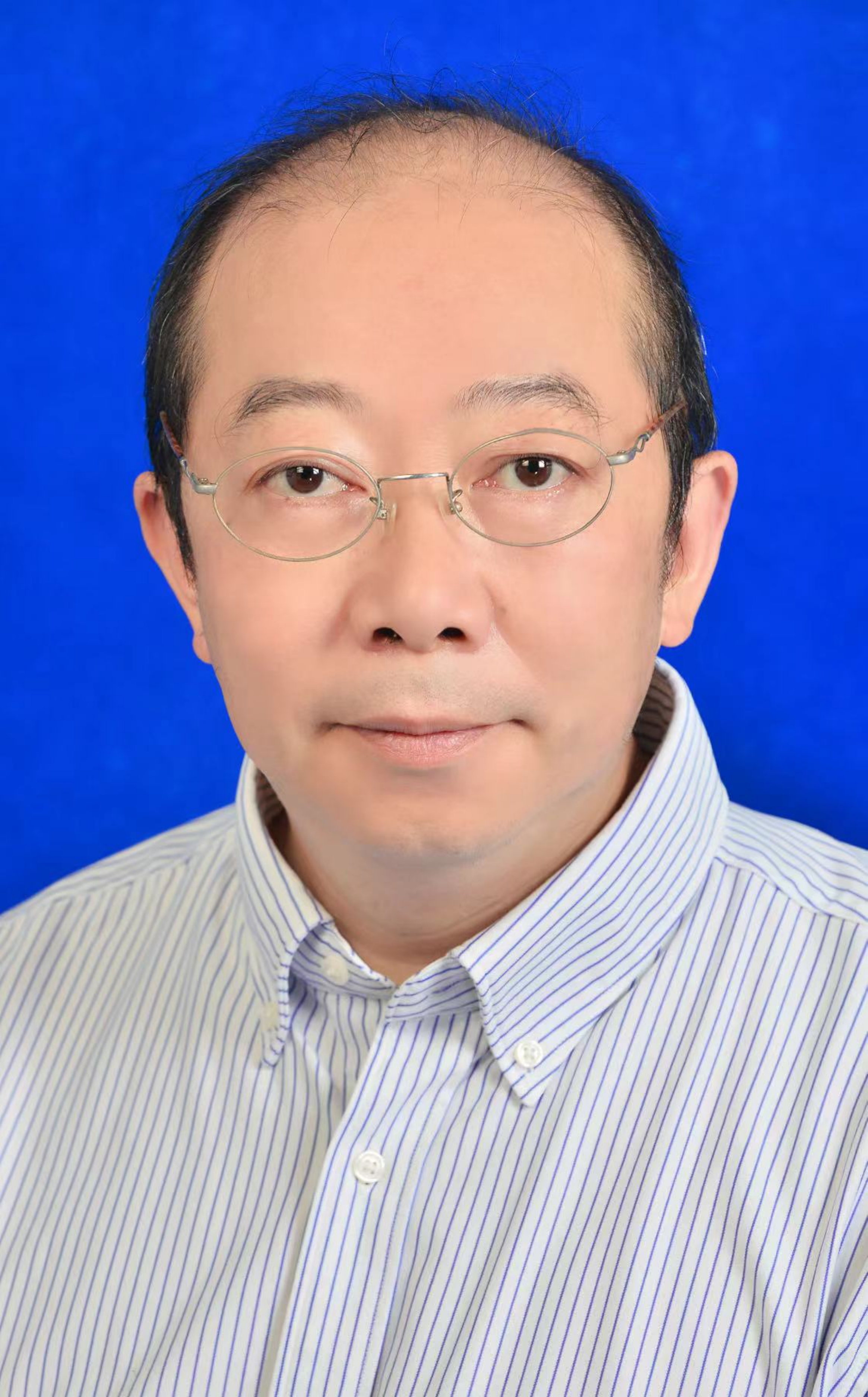}}]{Wenyu Chen}
received the Ph.D. degree in computer science from the University of Electronic Science and Technology of China, Chengdu, China, in 2009. He is currently a Professor with the School of Computer Science and Engineering, University of Electronic Science and Technology of China. His research work focuses on pattern recognition, natural language processing, and neural network.
\end{IEEEbiography}
\begin{IEEEbiography}[{\includegraphics[width=1in,height=1.25in,clip,keepaspectratio]{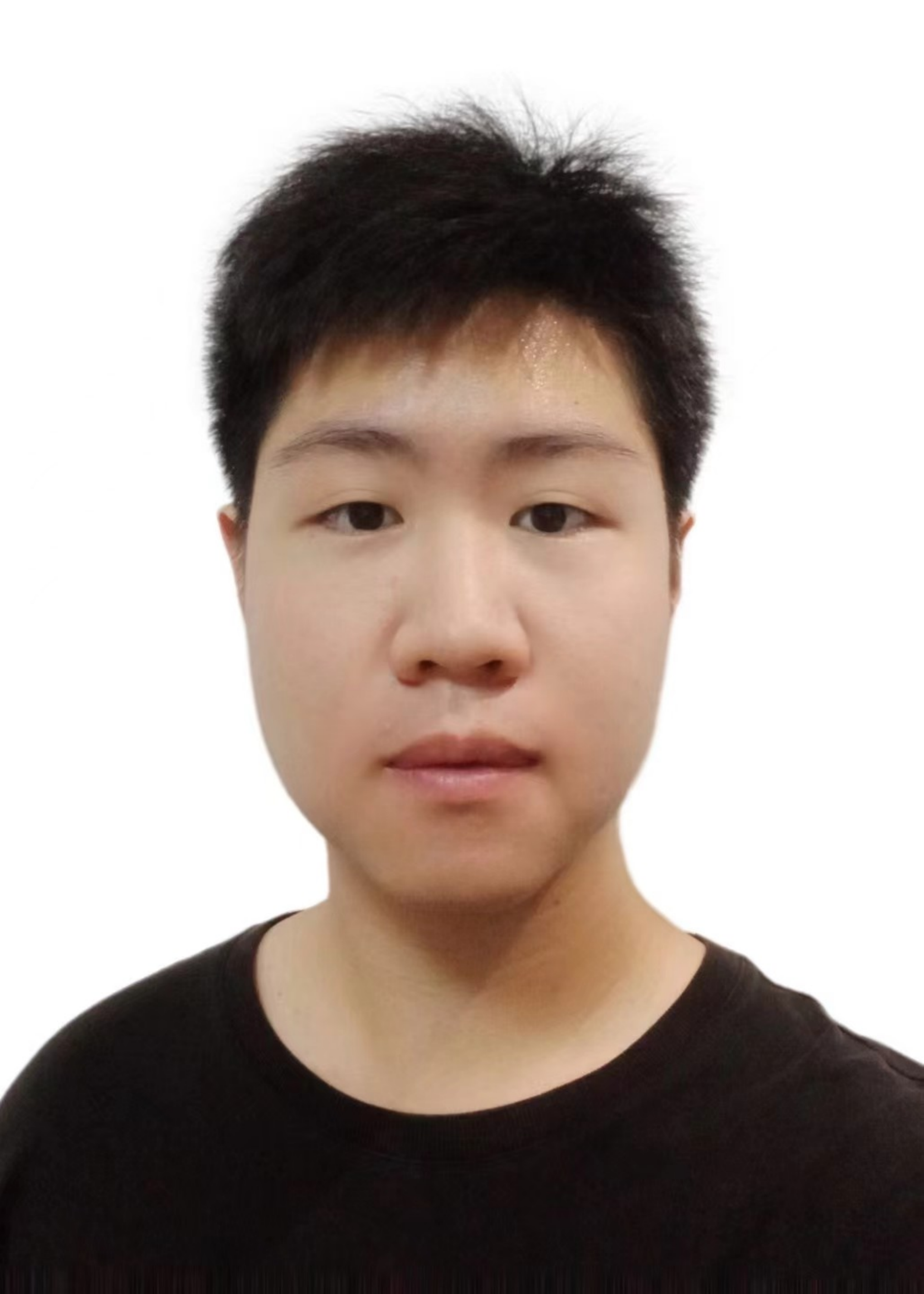}}]{Bingheng Li}
received his undergraduate degree in 2024 from the University of Electronic Science and Technology of China. He is currently pursuing a Ph.D. degree at the Michigan State University, USA. His research focuses on graph data mining and graph foundation models.
\end{IEEEbiography}

\end{document}